\documentclass[letterpaper]{article} 
\usepackage{aaai2026}  
\usepackage{times}  
\usepackage{helvet}  
\usepackage{courier}  
\usepackage[hyphens]{url}  
\usepackage{graphicx} 
\usepackage{natbib}  
\usepackage{caption} 
\usepackage{nicematrix}
\usepackage{algorithm}
\usepackage{algorithmic}
\usepackage{comment}

\usepackage{newfloat}
\usepackage{listings}
\DeclareCaptionStyle{ruled}{labelfont=normalfont,labelsep=colon,strut=off} 
\floatstyle{ruled}
\newfloat{listing}{tb}{lst}{}
\floatname{listing}{Listing}
\usepackage{graphicx}
\usepackage{amsmath}
\usepackage{amssymb}
\usepackage{mathtools}
\usepackage{amsthm}
\usepackage{mathbbol}
\usepackage{multirow}
\usepackage{graphicx}
\usepackage{subfigure}
\usepackage{booktabs} 
\usepackage{subfigure}
\usepackage{booktabs} 
\usepackage{dsfont}
\usepackage{mathtools}
\usepackage{amsfonts}
\usepackage{xcolor}
\usepackage{varwidth}
\usepackage{amsmath}
\usepackage{amsthm}
\usepackage{mathtools,amssymb}
\usepackage{lscape}
\usepackage{enumitem}
\usepackage{stmaryrd}

\newtheorem{theorem}{Theorem}
\newtheorem{theoreminappendix}{Theorem (Re)}
\newtheorem{lemma}{Lemma}
\newtheorem{lemmainappendix}{Lemma (Re)}
\newtheorem{remark}{Remark}

\newtheorem{corollary}{Corollary}

\def\E{\mathbb{E}}

\def\P{\mathbb P}

\def\ConfTr{\text{CTr}}

\def\UA{\text{CUT}}

\def\Sigmoid{{\text{Sigmoid}}}

\def\indicator{\mathbb{1}}

\def\APS{\text{APS}}
\def\HPS{\text{HPS}}
\def\RAPS{\text{RAPS}}
\def\SAPS{\text{SAPS}}
\def\reg{\text{reg}}

\def\calY{\mathcal Y}

\def\calC{\mathcal C}

\def\calD{\mathcal D}

\def\calP{\mathcal P}
\def\calQ{\mathcal Q}
\def\calB{\mathcal B}
\def\calH{\mathcal H}

\def\calL{\mathcal L}

\def\E{\mathbb E}
\def\P{\mathbb P}

\def\sup{\text{sup}}
\def\max{\text{max}}
\def\tr{\text{tr}}

\def\test{\text{test}}

\newcounter{checksubsection}
\newcounter{checkitem}[checksubsection]

\newcommand{\checksubsection}[1]{%
  \refstepcounter{checksubsection}%
  \paragraph{\arabic{checksubsection}. #1}%
  \setcounter{checkitem}{0}%
}

\newcommand{\checkitem}{%
  \refstepcounter{checkitem}%
  \item[\arabic{checksubsection}.\arabic{checkitem}.]%
}
\newcommand{\question}[2]{\normalcolor\checkitem #1 #2 \color{blue}}
\newcommand{\ifyespoints}[1]{\makebox[0pt][l]{\hspace{-15pt}\normalcolor #1}}

\title{ Cost-Sensitive Conformal Training with Provably Controllable Learning Bounds }
\author{
    Xuesong Jia\equalcontrib\textsuperscript{\rm 1},
    Yuanjie Shi\equalcontrib\textsuperscript{\rm 1},
    Ziquan Liu\textsuperscript{\rm 2}, 
    Yi Xu\textsuperscript{\rm 3}, 
    Yan Yan\textsuperscript{\rm 1}
}
\affiliations{
    \textsuperscript{\rm 1}School of EECS, Washington State University\\

    \textsuperscript{\rm 2}School of EECS, Queen Mary, University of London\\
    \textsuperscript{\rm 3}Dalian University of Technology\\
    \{xuesong.jia, yuanjie.shi, yan.yan1\}@wsu.edu, 
    ziquanliu.cs@gmail.com,
    yxu@dlut.edu.cn
}

\usepackage{bibentry}

\begin{document}

\maketitle

\begin{abstract}
Conformal prediction (CP) is a general framework to quantify the predictive uncertainty of machine learning models that uses a set prediction to include the true label with a valid probability.
To align the uncertainty measured by CP, conformal training methods minimize the size of the prediction sets.
A typical way is to use a surrogate indicator function, usually Sigmoid or Gaussian error function.
However, these surrogate functions do not have a uniform error bound to the indicator function, leading to uncontrollable learning bounds.
In this paper, we propose a simple cost-sensitive conformal training algorithm that does not rely on the indicator approximation mechanism.
Specifically, we theoretically show that minimizing the expected size of prediction sets is upper bounded by the expected rank of true labels.
To this end, we develop a rank weighting strategy that assigns the weight using the rank of true label on each data sample.
Our analysis provably demonstrates the tightness between the proposed weighted objective and the expected size of conformal prediction sets.
Extensive experiments verify the validity of our theoretical insights, 
and superior empirical performance over other conformal training in terms of predictive efficiency with $21.38\%$ reduction for average prediction set size.
\end{abstract}

\begin{links}
    \link{Code}{https://github.com/JoSaitama/RWCE}
\end{links}

\section{ Introduction }
\label{section:introduction}

\begin{figure}[!t]
    \centering
    \begin{minipage}[t]{0.49\linewidth}
    \centering
    \textbf{(a)} Comparison on ResNet
    \end{minipage} 
    \begin{minipage}[t]{0.49\linewidth}
    \centering
    \textbf{(b)} Comparison on DenseNet
    \end{minipage} 
    \hfill
    \begin{minipage}[t]{0.49\linewidth}
     \centering   
     \includegraphics[width = \linewidth]{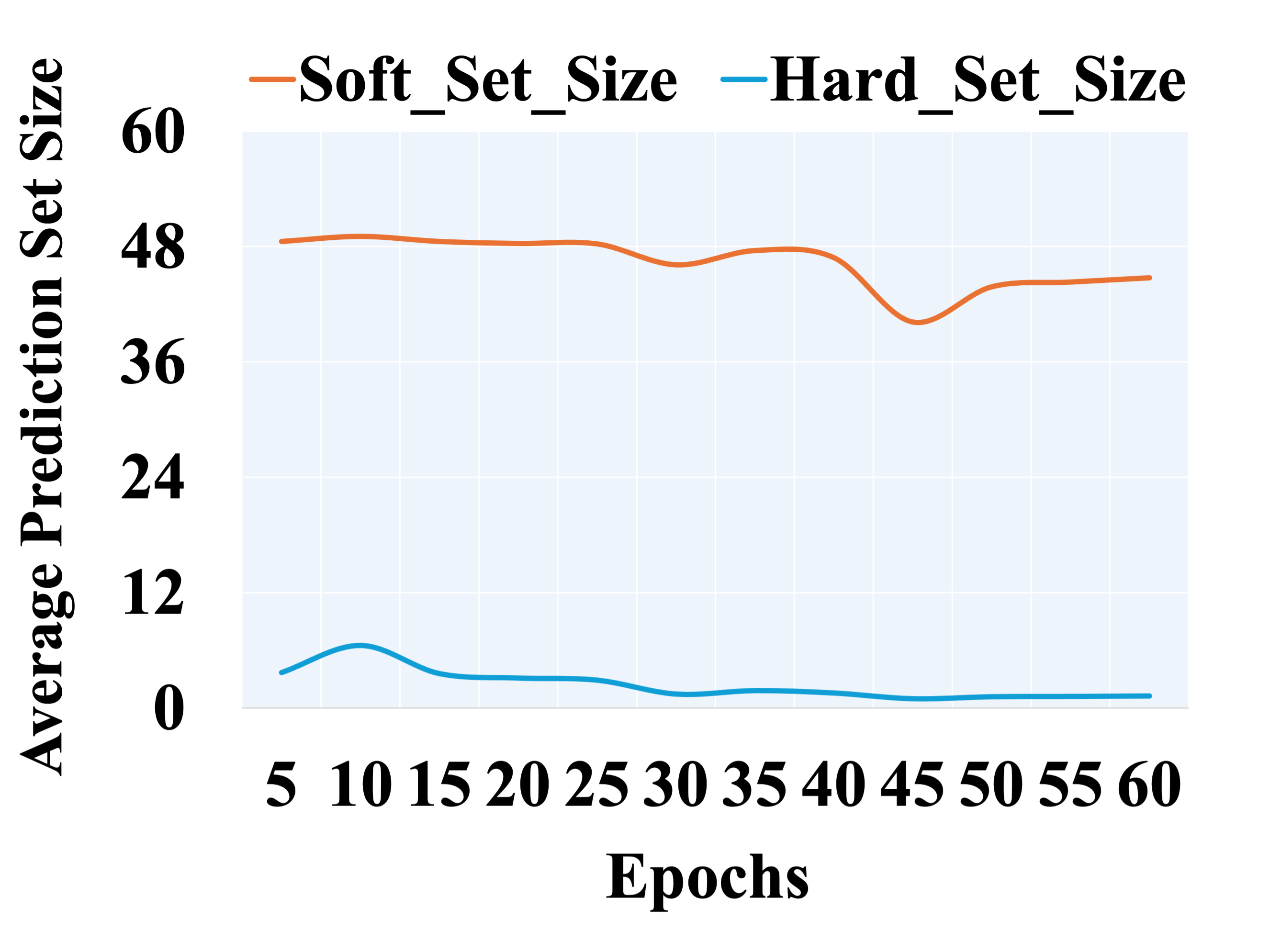}
     \end{minipage}
    \begin{minipage}[t]{0.49\linewidth}
    \centering
    \includegraphics[width=\linewidth]{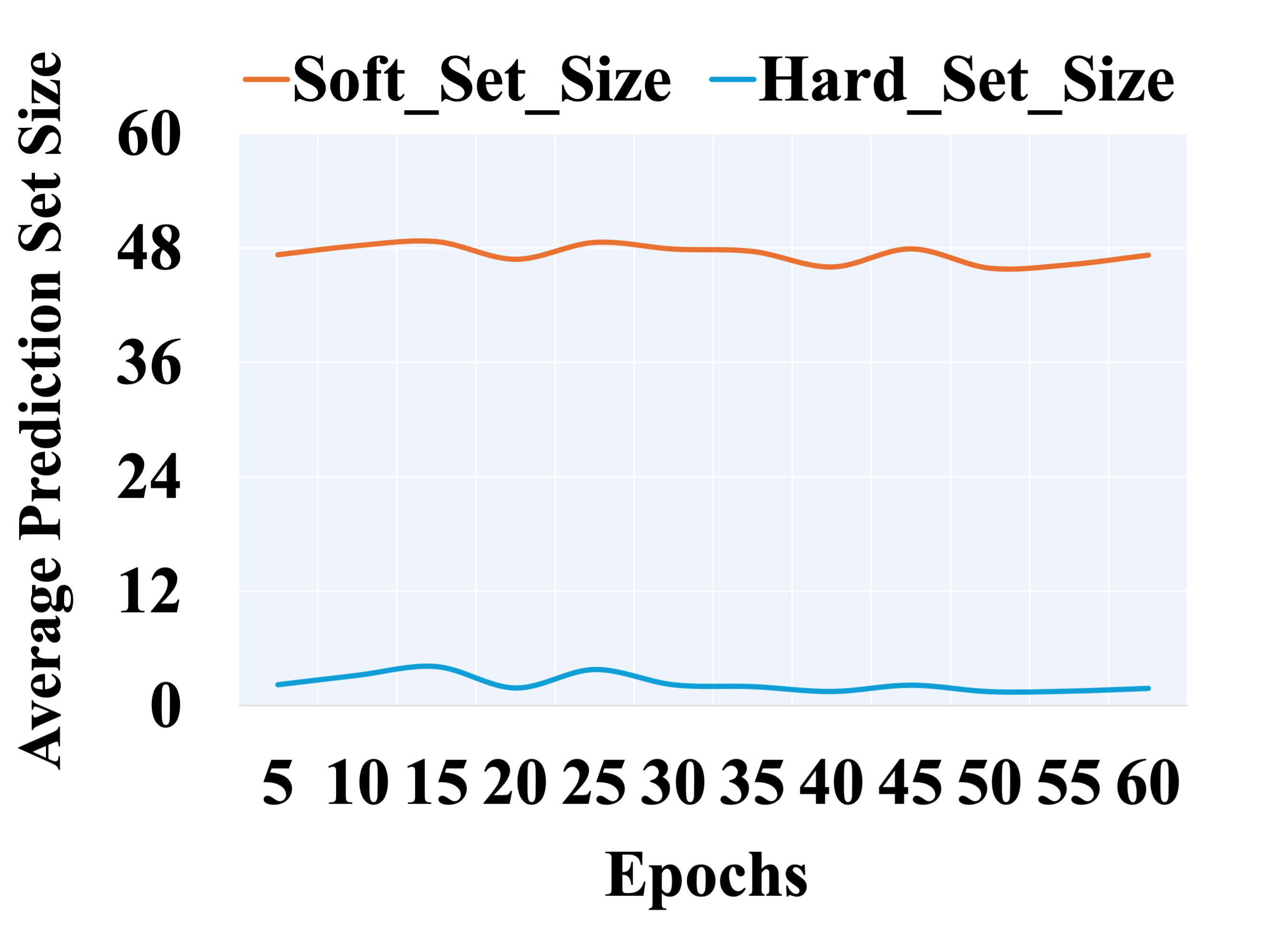}
    \end{minipage}
    \caption{
    \textbf{Large approximation error of the learning objective (soft set size) from the true objective (hard set size) during training of ConfTr \cite{stutz2021learning}} on CIFAR-100 using ResNet (a) and DenseNet (b).  
    The learning objective is the  differentiable measure of prediction set sizes by smooth approximations of the indicator function (e.g., Sigmoid), 
    while the hard set size reflects actual true objectives based on the indicator function.  
    The large and persistent gap between the two highlights the loose approximation introduced by approximated functions, 
    supporting our motivation to avoid indicator approximations and directly optimize a tight upper bound on the true prediction set size.
    }
    \label{fig:conftr_gap}
\end{figure}

Uncertainty Quantification (UQ) \cite{abdar2021review} is critical for the safe and reliable deployment of modern deep learning models, particularly in high-stakes domains such as medical diagnosis~\cite{begoli2019need,yang2021uncertainty} and autonomous driving~\cite{shafaei2018uncertainty,bachute2021autonomous}. 
In these safety-critical settings, unreliable or missing uncertainty estimates can lead to decisions made without calibrated confidence, increasing the risk of costly errors or harmful outcomes.
UQ methods aim to address this challenge by providing confidence-aware predictions that quantify uncertainty alongside predictions.  
Among UQ methods, Conformal Prediction (CP) \cite{vovk2005algorithmic,shafer2008tutorial,angelopoulos2021gentle} has emerged as a principled, distribution-free framework for constructing reliable prediction sets with formal statistical coverage guarantees.

CP is a powerful and general framework for UQ that transforms the output of a machine learning (ML) model into a prediction set with a user-specified coverage guarantee, typically $1 - \alpha$, without requiring assumptions on the data distribution or model class \cite{vovk2005algorithmic,angelopoulos2021gentle}. 
By leveraging a held-out calibration set, CP ensures that the generated prediction sets contain the true label with high probability, even in finite-sample regimes \cite{romano2020classification,angelopoulos2020uncertainty}. Owing to its broad applicability and strong theoretical guarantees, CP has found increasing use in critical applications such as fairness-aware learning and robust decision-making \cite{lu2022fair,lu2022improving}. 
However, a key limitation of standard CP methods is their reliance on post-hoc calibration where the training objective of classification model is not minimizing prediction set size, which often results in overly conservative (i.e., large) prediction sets that limit practical utility \cite{babbarijcai2022utility,straitouriicml2023improving}. 

Recent work~\cite{stutz2021learning,einbinder2022training,shi2025direct,kiyani2024length} directly incorporates the {\it conformal principles},
e.g., minimizing the size of prediction sets,
into model training to produce tighter and more informative prediction sets while preserving valid coverage. 
Rather than treating calibration as a separate post-hoc step,
these methods embed conformal objectives into the learning process (which we refer to as {\it conformal training}), 
enabling the model to learn representations that are inherently aligned with uncertainty quantification goals.

One common approach of conformal training is to replace the discrete prediction set size with a differentiable surrogate.
It is often referred to as the soft set size and constructed by smooth approximations of the indicator function, e.g., the Sigmoid \cite{stutz2021learning} or Gaussian error function \cite{kiyani2024length}. 
While these relaxations enable gradient-based optimization, 
they introduce an inherent mismatch: 
these continuous surrogates cannot precisely approximate the indicator, 
and do not provide a uniform approximation error bound. 
As a result, the training loss may deviate significantly from the true prediction set size, leading to suboptimal or unstable behavior during optimization.
Figure \ref{fig:conftr_gap} illustrates this effect on CIFAR‑100 with both ResNet and DenseNet backbones: 
the optimized soft set size (learning objective) remains dramatically separated from the actual hard set size (true objective) throughout training.
This raises a natural question: \emph{how could we develop a conformal training method that directly optimizes a tight, theoretically controlled upper bound on the true prediction‑set size and dispense with loose surrogates altogether?}

In this paper, we propose a simple yet theoretically grounded cost-sensitive conformal training algorithm that eliminates the need for surrogate approximations of the prediction set indicator. 
Instead of relying on smooth relaxations, we design a differentiable training objective that leverages the true-label rank to align more closely with the goal of conformal prediction. 
Our key insight is that the expected size of conformal prediction sets can be tightly upper bounded by the expected rank of the true label under the predictive scores of model. 
Based on this, we formulate a rank-weighted cross-entropy (RWCE) loss, where each sample is reweighted according to its label rank, yielding a simple rank-weighting scheme that directly incentivizes tighter prediction sets. 
This objective is efficient to compute, easy to implement with standard gradient-based optimization, and requires no additional smoothing parameters. 
We theoretically prove that our proposed loss upper bounds both the expected label rank and, consequently, the expected prediction set size, under mild and empirically verifiable conditions. 
Furthermore, we establish a generalization bound that ensures the empirical loss remains close to its population counterpart with high probability. 
Experiments on multiple benchmark datasets demonstrate that our method produces significantly smaller prediction sets with an average $21.38\%$ reduction while preserving valid coverage.

\noindent {\bf Contributions.} The key contributions of this paper include:
\begin{itemize}
     \item We propose a simple yet effective cost-sensitive conformal training algorithm that avoids surrogate indicator approximations. 
     Our method minimizes a RWCE loss, where the importance weight is derived from the true-label rank, providing a differentiable and practical objective for reducing prediction set size.
     \item We theoretically show that the expected prediction set size is upper bounded by the expected label rank, and that our objective tightly upper bounds this rank under mild conditions. 
     We also establish generalization bounds for RWCE loss, ensuring reliable performance in finite-sample settings.
    \item Experiments on multiple benchmarks demonstrate that our method consistently produces smaller prediction sets than existing conformal training baselines with an average $21.38\%$ reduction, while maintaining valid coverage guarantees.
\end{itemize}

\section{ Related Work }
\label{section:related_work}

\paragraph{Conformal Prediction (CP).} 
CP \cite{angelopoulos2021gentle,vovk2005algorithmic,shafer2008tutorial} is a distribution-free, post-hoc framework for uncertainty quantification that constructs prediction sets with user-defined coverage guarantees by calibrating conformity scores, which quantify how typical a new instance is relative to a labeled reference set. This framework provides rigorous, finite-sample coverage guarantees. CP can be successfully applied to a wide range of tasks, including regression \cite{romano2019conformalized,gibbs2023conformal,gibbs2021adaptive}, classification \cite{romano2020classification}, computer vision \cite{bates2021distribution}, federated learning \cite{lu2023federated,humbert2023one,plassier2023conformal,plassier2024efficient}, online learning \cite{gibbs2021adaptive,xi2025robust}, and adversarial perturbation \cite{baoenhancing,ghosh2023probabilistically}. Its performance is primarily evaluated by two critical and often competing criteria: valid coverage and predictive efficiency \cite{angelopoulos2021uncertainty}. Recently, research in CP has progressed along two major directions: (i) designing novel nonconformity scores and calibration procedures  \citep{angelopoulos2021uncertainty,huang2023conformal,fisch2021few,fisch2021efficient,guan2023localized,ding2024class,kiyani2024length,shi2024conformal,zhu2025predicate,luo2025reliable}, and (ii) developing conformal training techniques \cite{bellotti2021optimized,stutz2021learning,einbinder2022training,yan2024provably,shi2025direct}, which embed CP principles directly into the learning process to jointly optimize coverage and prediction set size. 

\paragraph{Conformal training.} 
Conformal training, initially proposed by Stutz et al. \cite{stutz2021learning}, is an emerging direction that aims to integrate the statistical guarantees of conformal prediction (CP) directly into the model training phase. Unlike standard CP approaches that apply calibration as a separate post-hoc procedure, conformal training incorporates conformity-based objectives into the training loop itself \cite{bellotti2021optimized,stutz2021learning,einbinder2022training,yan2024provably}. 
A central challenge in conformal training lies in formulating a differentiable surrogate for the prediction set size, which is inherently discrete and non-differentiable. 
To address this, existing approaches such as ConfTr~\cite{stutz2021learning}, CPL~\cite{kiyani2024length}, and DPSM \cite{shi2025direct} adopt soft approximations (e.g., Sigmoid or Gaussian error functions) to relax the indicator function that defines whether a label is included in the prediction set. 
These surrogates allow the use of gradient-based optimization but come at the cost of introducing significant approximation error.
As a result, the surrogate objective may poorly reflect the true prediction set behavior, leading to conservative prediction sets, limiting the practical benefits of conformal training. 
Consequently, despite the promise of integrating CP into end-to-end learning, current surrogate-based approaches struggle to faithfully minimize the quantity that matters most: the size of the conformal prediction set under the coverage constraint. 
These limitations motivate the need for more principled objectives that better capture the structure of conformal prediction without relying on soft approximations.

\section{Problem Setup and Motivation}
\label{section:problem_setup}

\textbf{Notations.}
Let $X \in \mathcal{X}$ denote an input instance and $Y \in \mathcal{Y} = \{1, 2, \ldots, K\}$ represent its corresponding true class label, where $K$ is the total number of classes.
Assume that each pair $(X, Y)$ is drawn i.i.d. from a joint distribution $\mathcal{P}$ over the space $\mathcal{X} \times \mathcal{Y}$. The training dataset is denoted as $\mathcal{D}_{\text{train}} = \{(X_i, Y_i)\}_{i=1}^n$ consisting of $n$ i.i.d. samples.
Additionally, we further assume access to two disjoint datasets: a calibration set $\mathcal{D}_{\text{cal}}$ and a test set $\mathcal{D}_{\text{test}}$, where $|\mathcal{D}_{\text{cal}}| = m$.
We use $\mathbf{1}[\cdot]$ to denote the indicator function.
Additionally, we define a mini-batch $\mathcal{B} = \{(X_i, Y_i) \mid i \in I_s\}$ of size $s$ by sampling indices $I_s \subset \{1, 2, \ldots, n\}$ with $|I_s| = s$.
We consider a probabilistic classifier $f \in \mathcal{F}$ that maps inputs to the probability simplex, i.e., $f(X): \mathcal{X} \rightarrow \Delta_+^K$, where $\Delta_+^K$ denotes the $(K{-}1)$-dimensional simplex.
The output $f(X)_y$ indicates the predicted confidence (e.g., Softmax score) assigned to class $y$.
Here, $\mathcal{F}$ refers to the hypothesis class from which the classifier is selected.
The label rank of $y$ on $X$ predicted by the ML model $f$ is defined by $R_f(X, y) = \sum_{l=1}^K \indicator[ f(X)_l \geq f(X)_y ]$.
We omit the subscript $f$ when unambiguous.

{\bf Conformal Prediction.}
Conformal prediction (CP) adjusts the outputs of a pre-trained model using a {\it nonconformity} scoring function.
We define a nonconformity score function as $S: \mathcal{X} \times \mathcal{Y} \rightarrow \mathbb{R}$, which quantifies how typical a candidate label is for a given input \cite{vovk2005algorithmic}.
For any training example $(X_i, Y_i) \in \mathcal{D}_{\text{cal}}$, we write its nonconformity score as $S_{f,i} = S_f(X_i, Y_i)$.
Let $S_{f,(j)}$ denote the $j$-th order statistic (i.e., the $j$-th smallest value) in the set $\{S_{f,i}\}_{i=1}^m$. A variety of scoring functions have been introduced in the literature \cite{shafer2008tutorial, angelopoulos2021uncertainty, huang2023conformal}.
In this work, we focus on several representative approaches: Homogeneous Prediction Sets (HPS) \cite{sadinle2019least}, Adaptive Prediction Sets (APS) \cite{romano2020classification}, Regularized Adaptive Prediction Sets (RAPS) \cite{angelopoulos2021uncertainty}, and Sorted Adaptive Prediction Sets (SAPS) \cite{huang2024conformal}.
We assume that the nonconformity scores are distinct (i.e., no ties), following \cite{romano2020classification}.

Given a predefined miscoverage rate $\alpha \in (0,1)$ and a test input-label pair $(X_{\text{test}}, Y_{\text{test}})$, the goal of CP is to construct a prediction set $\mathcal{C}_f: \mathcal{X} \rightarrow 2^{\mathcal{Y}}$ such that the true label is contained in the set with high probability:
\begin{align}
\label{eq:cp_coverage}
\P_{(X_\test, Y_\test)\backsim \calP} \{ Y_{\test} \in \calC_f(X_{\test}) \} 
\geq 
1- \alpha
.
\end{align}

To achieve this guarantee, the prediction set is formed by thresholding nonconformity scores using a quantile statistic estimated from a held-out calibration dataset $\mathcal{D}_{\text{cal}}$:
\begin{align*}
\widehat C_f(X) = \{ y \in \calY : S_f(X, y) \leq \widehat Q \}
\end{align*}
where $\widehat Q$ is determined by the empirical quantile computed as the $\lceil (1 - \alpha)(m + 1) \rceil$-th smallest value among $\{ S_{f,i} \}_{i=1}^m$, with $m = |\mathcal{D}_{\text{cal}}|$.


\textbf{Conformal training.}
Conformal training embeds the desiderata of conformal prediction directly into the learning objective of a classifier \(f\), so that the model is optimized not only for point-wise accuracy but also for the quality of the prediction sets it will output~\cite{bellotti2021optimized,stutz2021learning,einbinder2022training,shi2025direct}. 
Among several proposed conformal losses (see Appendix~\ref{appendix:section:additional_details}), we focus on the particularly natural objective that penalizes the expected size of the conformal prediction set:
\begin{align}
\label{eq:conformal_training_conftr_population}
\calL_c(f) 
\! = \! \E_X[|\calC_f(X)|] 
\! = \! \E_X \Big [ \sum_{y \in \calY} \indicator \big [ S_f(X, y) \! \leq Q \! \big ] \Big ]
.
\end{align}

As the indicator in~\eqref{eq:conformal_training_conftr_population} is non-differentiable, existing methods, such as ConfTr \cite{stutz2021learning} and CPL \cite{kiyani2024length}, replace it with a smooth surrogate, e.g., a Sigmoid function.
Then the effective conformal loss $\widetilde \calL_c$ is built by 
\begin{align}
\label{eq:conformal_training_conftr_SA}
\widetilde \calL_c(f) 
= &
\E_{X} \Big [ \sum_{y \in \calY } \tilde \indicator \big [ S_f(X, y) \leq Q \big ] \Big ]
,
\end{align}
where $\tilde \indicator[\cdot]$ is a smoothed estimator for the indicator function $\indicator[\cdot]$ and defined by a Sigmoid function \cite{stutz2021learning}, i.e.,
$\tilde \indicator[ S \leq q ] = 1/(1+\exp(-(q-S)/\tau))$, or Gaussian error function \cite{kiyani2024length}, i.e.,
$\tilde \indicator[ S \leq q ] = \frac{1}{2} \big (1 + \text{erf} ((q-S)/\sqrt{2} \sigma) \big )$, where $\text{erf} (a) = 2/\sqrt{\pi} \int^a_0 e^{- t^2}$ and $\sigma$ controls the smoothness.

Although smooth relaxations make the prediction set size differentiable, they fail to capture the binary decision structure inherent to CP. 
Specifically, prediction sets are defined by a hard thresholding rule: a label is either included or excluded based on whether its conformity score falls below a calibrated quantile. 
Smooth surrogates, such as the Sigmoid or Gaussian error function, can only approximate this decision boundary continuously, and inevitably introduce approximation error near the threshold. 
Since these surrogates lack uniform approximation guarantees, the training objective may diverge from the true prediction set size, resulting in misaligned gradients and suboptimal performance.
Recall that Figure~\ref{fig:conftr_gap}, this misalignment manifests empirically on CIFAR-100 with both ResNet and DenseNet backbones: the learning objective systematically overestimates the true objective across training epochs, resulting in a persistent optimization gap and an unreliable training signal.

\emph{Therefore, our goal is to propose a conformal training method which directly optimize a tight, differentiable, and controllable bound on the true prediction set size.}


\section{ Cost-Sensitive Conformal Training }
\label{section:method}

In this section, we introduce our proposed framework, \emph{Rank Weighted Cross-Entropy (RWCE)}, which incorporates label-rank information into the training objective to align with conformal prediction goals. We then provide a theoretical analysis of RWCE, establishing learning bounds that justify its ability to tightly control prediction set size while ensuring reliable generalization.

\subsection{Importance Weighting with Label Rank}
\label{subsection:iw_with_label_rank}

Our approach avoids loose indicator approximations by directly leveraging the rank of the true label, a discrete statistic that reflects how prominently the correct class is scored by the model. 
Intuitively, a model that consistently ranks the true label near the top is more likely to produce smaller prediction sets after calibration, as fewer incorrect classes will score above the threshold. 
Therefore, the expected rank naturally serves as a proxy for prediction set size.
The key insight of our work is to formalize and exploit this connection: by minimizing a rank-weighted cross-entropy loss, our method provides a differentiable and practical objective that aligns training directly with the efficiency goal of conformal prediction without relying on any surrogate relaxation.

To incorporate this idea into a differentiable training loss, we introduce an importance weighting scheme via the true-label rank. 
We rewrite the true label rank as $R_f(X, Y)$ to explicitly show the impact of model $f$, where it denotes the position of the true label \(Y\) in the sorted list of class scores.
Crucially, $R_f(X, Y)$ is treated as a fixed, non-differentiable weight: it serves purely as a cost signal, and we do not backpropagate through the rank itself.
Harder examples (those ranked higher) incur a larger penalty, while easier examples (those ranked lower) are rewarded. 
By scaling the standard cross-entropy loss with the rank, we obtain a cost-sensitive loss function that emphasizes reducing the frequency and magnitude of low-rank predictions. 
This weighting acts as a soft prioritization mechanism: samples with high confidence in the correct class are minimally penalized, while uncertain samples contribute more to the overall loss. 
Compared to other soft-set-size surrogates that apply uniform smoothing to all samples, our approach focuses optimization effort on correcting misranked examples in a way that is tightly coupled to conformal prediction efficiency.
We define the \emph{population loss} for the rank-weighted cross-entropy (RWCE) as:
\begin{align}
\label{eq:objective_populaton}
\calL (f)\triangleq \E_{(X,Y) \sim \calD} \left[ R_f(X, Y) \cdot \ell_f(X, Y) \right],
\end{align}
where \(\ell_f(X,Y)\) is the standard cross-entropy loss.

The corresponding empirical objective is:
\begin{align}
\label{eq:objective_empirical}
\widehat \calL (f)\triangleq \frac{1}{n} \sum_{i=1}^n R_f(X_i, Y_i) \cdot \ell_f(X_i, Y_i).
\end{align}

This loss is differentiable and avoids the surrogate approximations (e.g., Sigmoid function) employed in prior work for conformal training. As we show in the next subsection, this rank-weighted loss not only aligns closely with the conformal objective but also admits provable guarantees.

We optimize the objective in Eq.~\eqref{eq:objective_empirical} using stochastic gradient descent. 
The optimization algorithm is summarized in Algorithm~\ref{alg:CTLR}.
At each iteration $t$, 
we first sample a mini-batch $\calB_t$ and compute the softmax scores $f_t(X_j)$ for each sample in this batch, using the current model $f_t$ (see Line \ref{alg:line:sample_batch} and \ref{alg:line:compute_softmax}). 
Then, we compute the true label rank $R_{f_t}(X_j, Y_j)$ and cross-entropy loss $\ell_{f_t}(X_j, Y_j)$, respectively (see Line \ref{alg:line:compute_rank} and \ref{alg:line:compute_ce}).  
Next, we compute the RWCE loss $\widehat{\calL}(f_t)$ (see Line \ref{alg:line:compute_RWCE}).  
Finally, we update the model parameters $f_{t+1}$ (see Line \ref{alg:line:update_model}). 
Once training is finished, we output the trained classification model $f_T$ (Line \ref{alg:line:output}).
This procedure is simple, efficient (with only \(\mathcal{O}(K)\) cost to compute ranks), and introduces no additional hyperparameters such as temperature or smoothing terms in surrogate-based objectives.

\begin{algorithm}[t]
\caption{Conformal Training with Label Rank}
\label{alg:CTLR}
\begin{algorithmic}[1]

\STATE \textbf{Input}: Training dataset $\calD_\tr$, learning-rate $\eta>0$, batch size $s$

\STATE Randomly initialize the deep neural network $f_0$
\label{alg:line:initialize}

    \FOR{$t \leftarrow 0 : T-1$}

        \STATE Randomly sample a batch $\calB_t \subset \calD_\tr$
        \label{alg:line:sample_batch}

        \STATE Compute softmax scores $f_t(X_j)$ for all $X_j \in \calB_t$
        \label{alg:line:compute_softmax}

        \STATE Compute true label ranks $R_{f_t}(X_j, Y_j)$ for each $(X_j, Y_j)$
        \label{alg:line:compute_rank}

        \STATE Compute cross-entropy loss $\ell_{f_t}(X_j, Y_j)$ for each $(X_j, Y_j)$
        \label{alg:line:compute_ce}
        
       \STATE Compute RWCE loss:
    $
    \widehat{\calL}(f_t) \leftarrow \frac{1}{s} \sum_{j=1}^s R_{f_t}(X_j, Y_j) \cdot \ell_{f_t}(X_j, Y_j)
    $ on batch $\calB_t$
    \label{alg:line:compute_RWCE}
    
    \STATE Update model: $f_{t+1} \leftarrow f_t - \eta \cdot \nabla_f \widehat{\calL}(f_t)$
    \label{alg:line:update_model}
    
    \ENDFOR

\STATE \textbf{Output}: the classification model $f_T$
\label{alg:line:output}

\end{algorithmic}
\end{algorithm}

\subsection{Learning Bounds of Proposed Objective}
\label{subsection:learning_bounds}

We now provide theoretical justifications for using the RWCE loss. 
Our analysis proceeds in three parts: 
(i) we show that the expected rank upper bounds the expected size of conformal prediction sets, 
(ii) we show that our proposed objective upper bounds the expected rank under mild conditions, 
and (iii) we show that our empirical loss generalizes well with standard concentration bounds.


We begin by formalizing the connection between the true-label rank and the prediction set size. 
The following theorem shows that, under any conformity score and calibration procedure, the expected size of the conformal prediction set is upper bounded by the expected rank of the true label, up to a small calibration slack:
\begin{theorem}
\label{theorem:expected_rank_upper_bounds_expected_size}
For any nonconformity score $S$ and coverage level \( (1-\alpha) \in (0, 1) \), the following holds:
\begin{align}
\label{eq:theorem:expected_rank_upper_bounds_expected_size}
\E [ | \widehat C(X) | ]
\leq 
\E [ R(X, Y) ]
+ K \bigg( 1 - \alpha + \frac{1}{n+1} \bigg)
.
\end{align}
\end{theorem}

\begin{remark}
Theorem~\ref{theorem:expected_rank_upper_bounds_expected_size} implies that reducing the average
rank directly tightens the prediction set, with an
\(\mathcal{O}(K/n)\) calibration slack, which is a model-independent constant.
Although this slack may be numerically large, it does not
depend on the classifier and cannot be reduced through training.
Thus, the only optimizable part in the bound, expected rank, is a \emph{sound surrogate} for expected set size.
\end{remark}

However, directly minimizing the expected rank is still not ideal, as it is a discrete, non-differentiable quantity. 
Instead, we minimize our proposed objective: a rank-weighted version of the cross-entropy loss. 
To show this objective is sound, we now prove that, under a mild condition, it upper bounds the expected rank (up to a constant shift that does not affect optimization). 
This bridges the gap between the desired rank minimization and the tractable training loss.

\begin{theorem}
\label{theorem:new_objective_upper_bounds_expected_rank}
Suppose that $\E[(R(X, Y) - 1)(\ell(X, Y) - 1)] \geq -\E[\ell(X, Y)]$ or $\ell$ is the cross-entropy loss.
Then the following inequality holds:
\begin{align*}
\E[R(X, Y)-1]
\leq 
\E[ R(X, Y) \cdot \ell(X, Y)]
=
\calL(f)
.
\end{align*}
\end{theorem}

\begin{remark}
\label{remark:new_objective_bounds_rank}
This result shows that our objective \(\mathcal{L}(f)\) serves as an upper bound on the expected rank (up to a constant $1$), meaning that optimizing the RWCE loss leads to models with low expected rank and therefore small prediction sets by Theorem~\ref{theorem:expected_rank_upper_bounds_expected_size}. 
Importantly, the additive constant $1$ does not influence learning because it is independent of model parameters. 
Thus, our method can be viewed as optimizing a differentiable, rank-weighted variant of the cross-entropy loss. 
The covariance condition is mild and usually holds in practice.
It says that the rank prediction is positively correlated to the CE loss.
We validate this empirically in 
Figure~\ref{fig:rank_ce_alignment}.
\end{remark}

Theorem~\ref{theorem:expected_rank_upper_bounds_expected_size} establishes a fundamental connection between the expected prediction set size and the expected rank of the true label. 
This result formalizes the intuition that the more confidently a model scores the correct class (i.e., assigns it a high rank), the fewer classes are needed to meet the conformal coverage guarantee. 
In practice, this means that by designing models that minimize expected rank, we are implicitly driving down the prediction set size to achieve better predictive efficiency.

Theorem 2 strengthens this insight by showing that RWCE loss function directly upper bounds the expected rank under a mild and realistic correlation condition. 
This result provides strong theoretical justification for our proposed training objective: not only does it approximate the desired metric (expected rank), it does so in a provably tight and optimizable way. 


Finally, to ensure that our proposed training objective is learnable from finite data, we establish a standard generalization bound. This shows that the empirical version of our loss concentrates tightly around its population counterpart:
\begin{corollary}
\label{corrolary}
(Generalization error bound, \cite{mohri2018foundations})
\begin{align*}
\sup_{f \in \calH} ~ \big (\calL(f) - \widehat \calL(f) \big )
\leq 
\sqrt{ \frac{ \log(1/\delta) }{ n } }
.
\end{align*}
\end{corollary}

\begin{remark}
This result ensures that minimizing the empirical loss \(\widehat{\mathcal{L}}(f)\) effectively minimizes the population loss \(\mathcal{L}(f)\) with high probability, even in finite-sample settings. Together with Theorems~\ref{theorem:expected_rank_upper_bounds_expected_size} and~\ref{theorem:new_objective_upper_bounds_expected_rank}, this completes our theoretical justification for using RWCE for conformal training.
\end{remark}

\begin{table*}[!t]
\centering
\resizebox{\textwidth}{!}{
\begin{NiceTabular}{@{}c|cccc|cccc @{}}
\toprule
\multirow{2}{*}{Model} & \multicolumn{4}{c|}{HPS} & \multicolumn{4}{c}{APS} \\ 
\cmidrule(lr){2-5} \cmidrule(lr){6-9}
& \texttt{CE} & \texttt{CUT} & \texttt{ConfTr} & \texttt{RWCE} & \texttt{CE} & \texttt{CUT} & \texttt{ConfTr} & \texttt{RWCE}  \\ 
\midrule
\Block{1-*}{Caltech-101}
\\
\midrule
ResNet   
& 1.52 $\pm$  0.045 & 1.55 $\pm$ 0.039 & 1.26 $\pm$ 0.03 & \textbf{0.96 $\pm$ 0.008} ($\downarrow$ 23.81\%) 
& 4.96 $\pm$ 0.094 & 4.89 $\pm$ 0.095 & 4.25 $\pm$ 0.081 & \textbf{1.33 $\pm$ 0.017} ($\downarrow$ 68.71\%) 
\\ 
DenseNet 
& 3.51 $\pm$ 0.10  & 1.66 $\pm$ 0.038 & 3.13 $\pm$ 0.07  & \textbf{0.94 $\pm$ 0.005} ($\downarrow$ 43.37\%) 
& 8.92  $\pm$ 0.18 & 4.60 $\pm$ 0.078  & 9.69 $\pm$ 0.22  & \textbf{1.27 $\pm$ 0.008} ($\downarrow$ 72.39\%) 
\\ 

\midrule
\Block{1-*}{CIFAR-100}
\\
\midrule
ResNet   
& 3.39 $\pm$ 0.10  & 2.91 $\pm$ 0.08  & 3.98 $\pm$ 0.077  & \textbf{2.68 $\pm$ 0.083}  ($\downarrow$ 7.90\%) 
& 3.98 $\pm$ 0.13 & 3.49 $\pm$ 0.104 & 5.13 $\pm$ 0.117  & \textbf{3.08 $\pm$ 0.078} ($\downarrow$ 11.75\%)
\\ 
DenseNet 
& 2.59 $\pm$ 0.053 & 2.07 $\pm$ 0.06  & 2.19 $\pm$ 0.034  & \textbf{2.04 $\pm$0.051}  ($\downarrow$ 1.45\%) 
& 3.38 $\pm$ 0.12 & \textbf{2.19 $\pm$ 0.060} & 3.04 $\pm$ 0.069   &  2.335 $\pm$ 0.076 ($\uparrow$ 6.62\%)
\\ 

\midrule
\Block{1-*}{iNaturalist}
\\
\midrule
ResNet   
& 98.69 $\pm$ 8.86 & 73.19 $\pm$ 3.13 & 76.30 $\pm$ 3.12 & \textbf{69.05 $\pm$ 2.56} ($\downarrow$ 5.66\%) 
& 95.18 $\pm$ 3.50 & 79.58 $\pm$ 2.87 & 87.80 $\pm$ 1.97   & \textbf{73.39 $\pm$2.32} ($\downarrow$  7.78\%) 
\\ 
DenseNet 
& 93.18 $\pm$ 2.92 & 74.11 $\pm$ 2.35 & 72.25 $\pm$ 2.07 & \textbf{67.17 $\pm$ 2.07} ($\downarrow$ 7.03\%) 
& 101.55 $\pm$ 2.97 & 87.27 $\pm$ 2.30 & 92.88 $\pm$ 2.89 & \textbf{75.65 $\pm$ 2.44} ($\downarrow$ 13.31\%)
\\ 

\bottomrule
\end{NiceTabular}
}
\caption{ \textbf{Overall comparison on three datasets and calibrated by HPS and APS on DenseNet and ResNet} with $\alpha = 0.1$.
All models are evaluated under both HPS and APS calibration strategies to assess robustness across scoring functions.
We report the mean and standard deviation of the reported APSS values over 10 independent runs.
We benchmark four methods: standard CE, CUT, ConfTr, and  RWCE.
Arrows $\downarrow$ and $\uparrow$ indicate improvements or degradations in predictive efficiency relative to the best baseline.
Overall, RWCE consistently produces the smallest prediction sets across all datasets and evaluation metrics, demonstrating a relative improvement of $21.38\%$ in average prediction set size—highlighting its superior calibration quality and generalization capability.
}
\label{tab:cvg_set_hps_train_main_part}
\end{table*}

\section{ Experiments }
\label{section:experiments}

\subsection{Experiment Setup}
\label{subsection:experiment_setup}

\textbf{Datasets.} 
We firstly evaluate our method on three widely used image classification benchmarks with varying levels of label granularity and class imbalance: CIFAR-100~\cite{krizhevsky2009learning}, Caltech-101~\cite{fei2006one}, and iNaturalist \cite{van2018iNaturalist}. 
Each dataset is randomly partitioned into training, calibration, and testing subsets following standard conformal prediction procedures.  Table~\ref{tab:Data_stat} in Appendix summarizes the key statistics, including the number of classes, sample counts, and imbalance ratios. 
We also conduct experiments on SST-5 \cite{socher2013recursive}, a text classification benchmarks. 

\textbf{Backbone Models.} 
We adopt two widely used convolutional architectures, ResNet~\cite{he2016deep} and DenseNet~\cite{huang2017densely}, as deep models in our image classification experiments. 
Detailed hyperparameters for each configuration are provided in Appendix Table~\ref{tab:finetune_hyper_params}. 
For SST-5, we employed the BERT \cite{devlin2019bert} model.


\textbf{Nonconformity Scores.} 
We adopt four widely used nonconformity scoring functions in conformal prediction: HPS~\cite{vovk2005algorithmic, lei2013distribution}, APS~\cite{romano2020classification}, RAPS \cite{angelopoulosraps}, and SAPS \cite{huang2023uncertainty}. 
After training, each model is evaluated using HPS, APS, RAPS and SAPS based on the average prediction set size (APSS) at a fixed marginal coverage of 90\%.  
This multi-score evaluation framework provides a robust assessment of how well the model generalizes under different conformity measures.  
A detailed overview of these scoring functions is provided in Appendix~\ref{appendix:section:additional_details}.


\textbf{Baseline Methods.} 
We compare our approach against three representative baseline training methods:  
(i) CE, which performs standard empirical risk minimization using the cross-entropy loss without any uncertainty-aware regularization; 
(ii) CUT~\cite{einbinder2022training}, which penalizes the discrepancy between the empirical cumulative distribution function (CDF) of nonconformity scores and the uniform distribution to mitigate model overconfidence; and 
(iii) ConfTr~\cite{stutz2021learning}, which directly trains models to produce calibrated prediction sets by minimizing a differentiable approximation of the average set size.
Specifically, it employs a Sigmoid approximation of the indicator function to define a differentiable loss that jointly penalizes large prediction sets and miscoverage. 
For CUT and ConfTr, we use the HPS score for training.
The details of CUT and ConfTr are in Appendix \ref{appendix:section:additional_details}.


\textbf{Evaluation Metrics.} 
Our first evaluation metric is the average prediction set size (APSS), defined as $\text{APSS} = \frac{1}{|\calD_\test|}\sum_{i \in \calD_\test} |C_f({X}_{i})|$. 
In addition to efficiency, coverage is a fundamental metric in conformal prediction.
Accordingly, we report marginal coverage (Marg-Cov) as our second evaluation criterion, defined as $\text{Marg-Cov} = \frac{1}{|\calD_\test|}\sum_{i \in \calD_\test} \indicator[Y_{i} \in C_f({X}_{i})]$.

\subsection{Results and Discussion}

\begin{figure*}[!t]
    \centering
    \begin{minipage}[t]{0.32\linewidth}
    \centering
    \textbf{(a)}  Loss Convergence
    \includegraphics[width = \linewidth]{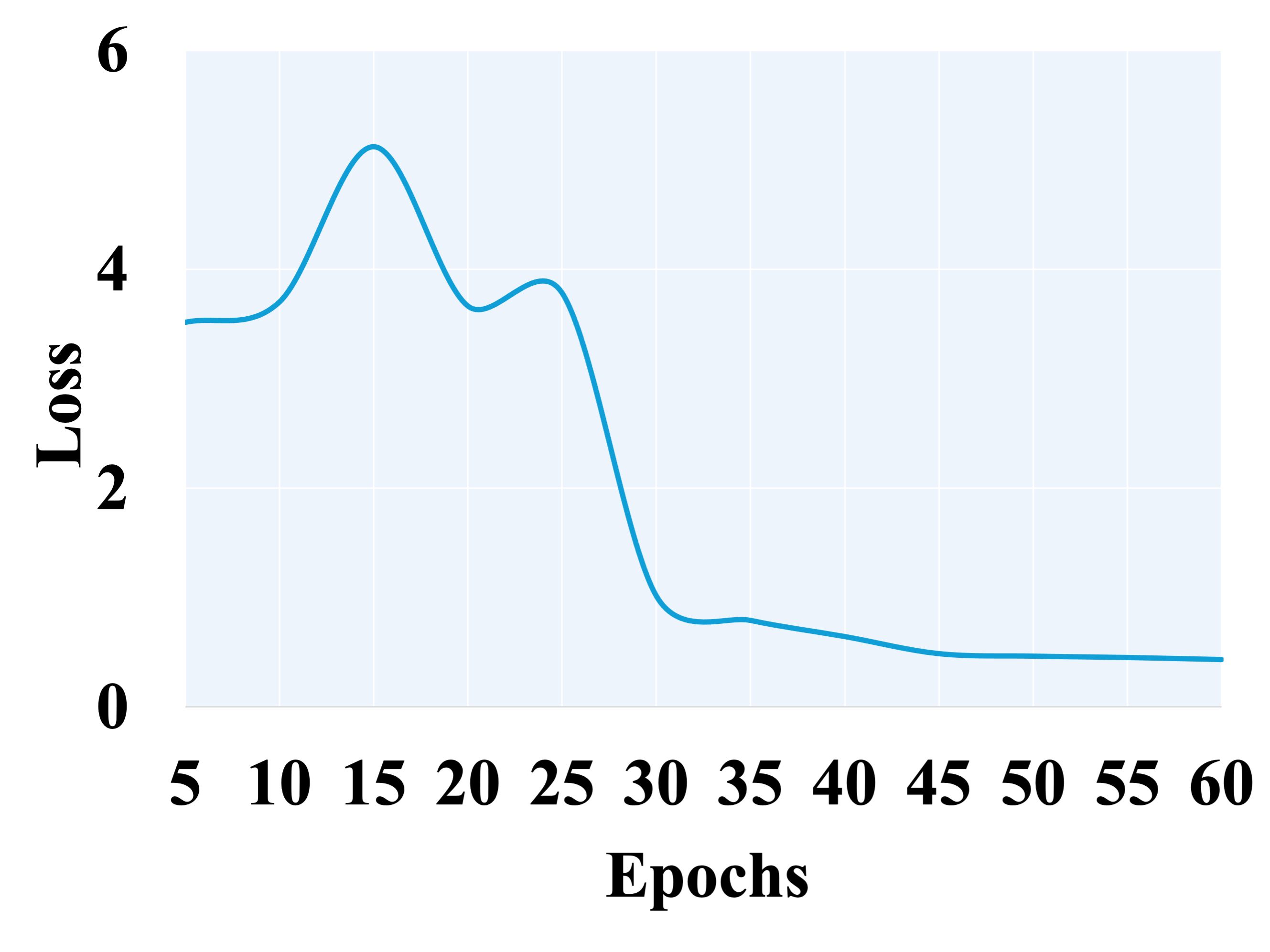}
    \\
    \includegraphics[width = \linewidth]{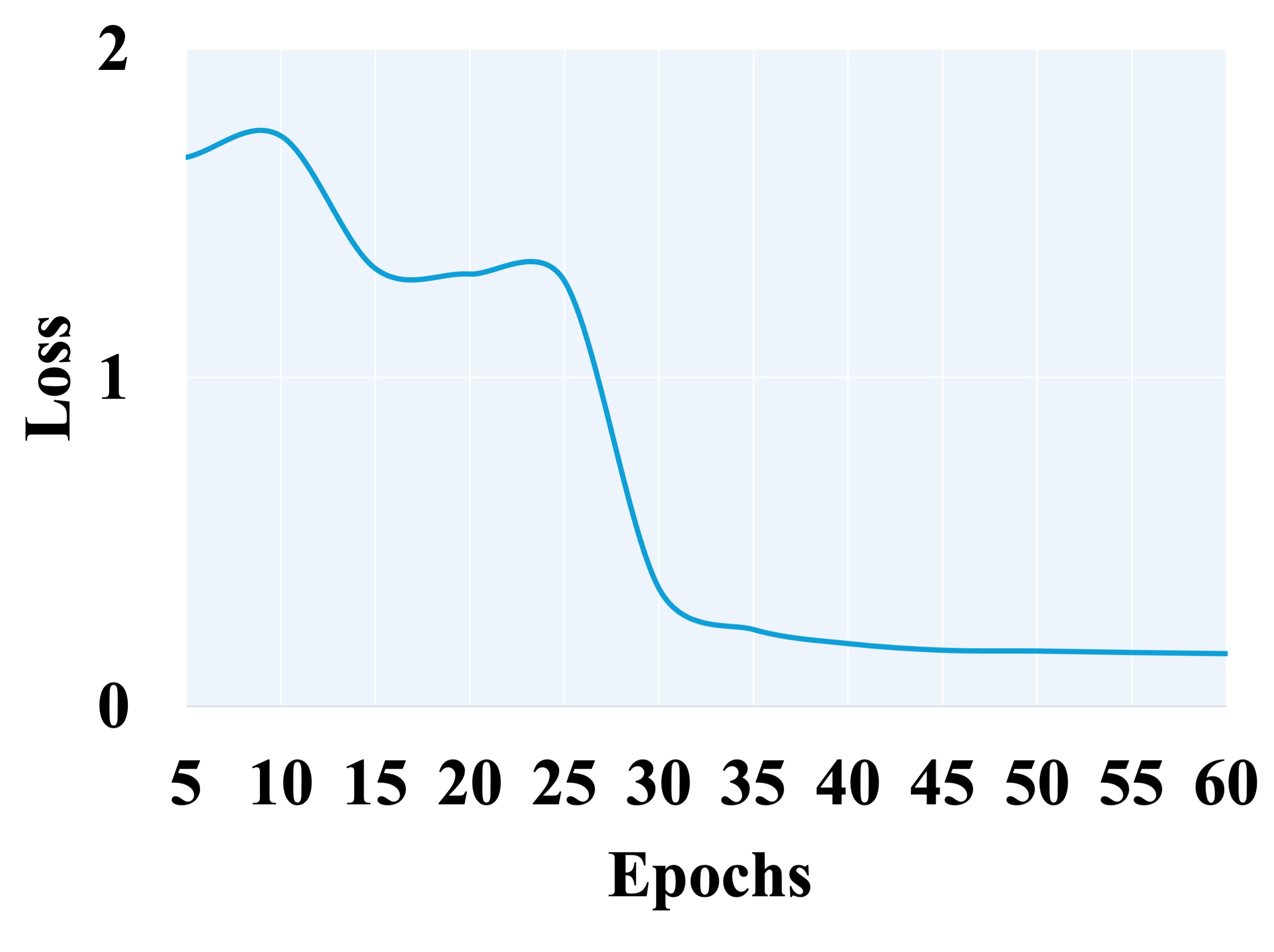}
    \end{minipage} 
    \begin{minipage}[t]{0.32\linewidth}
    
    \centering    
    \textbf{(b)} Loss vs APSS
    \includegraphics[width=\linewidth]{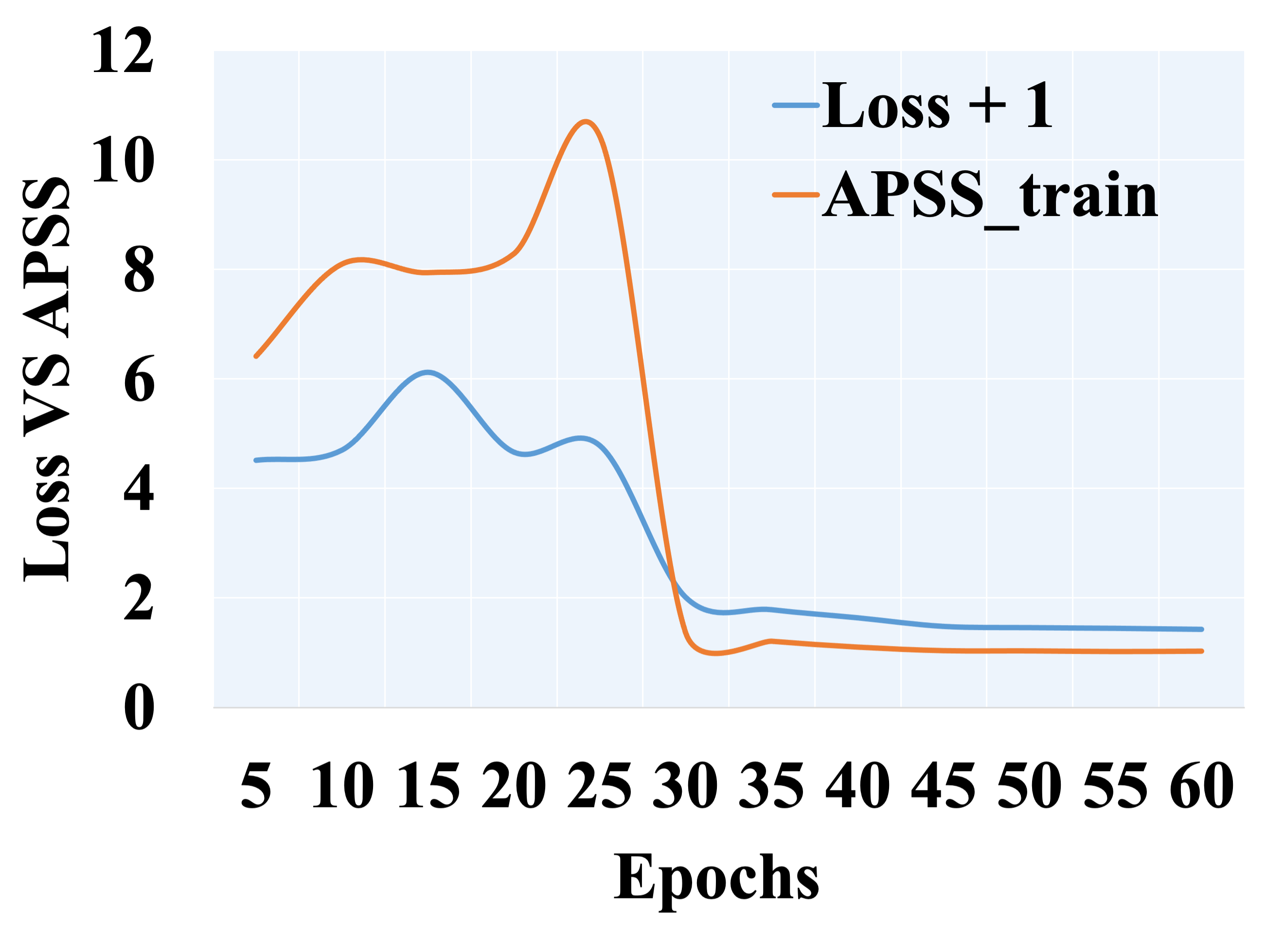}
    \\
    \includegraphics[width=\linewidth]{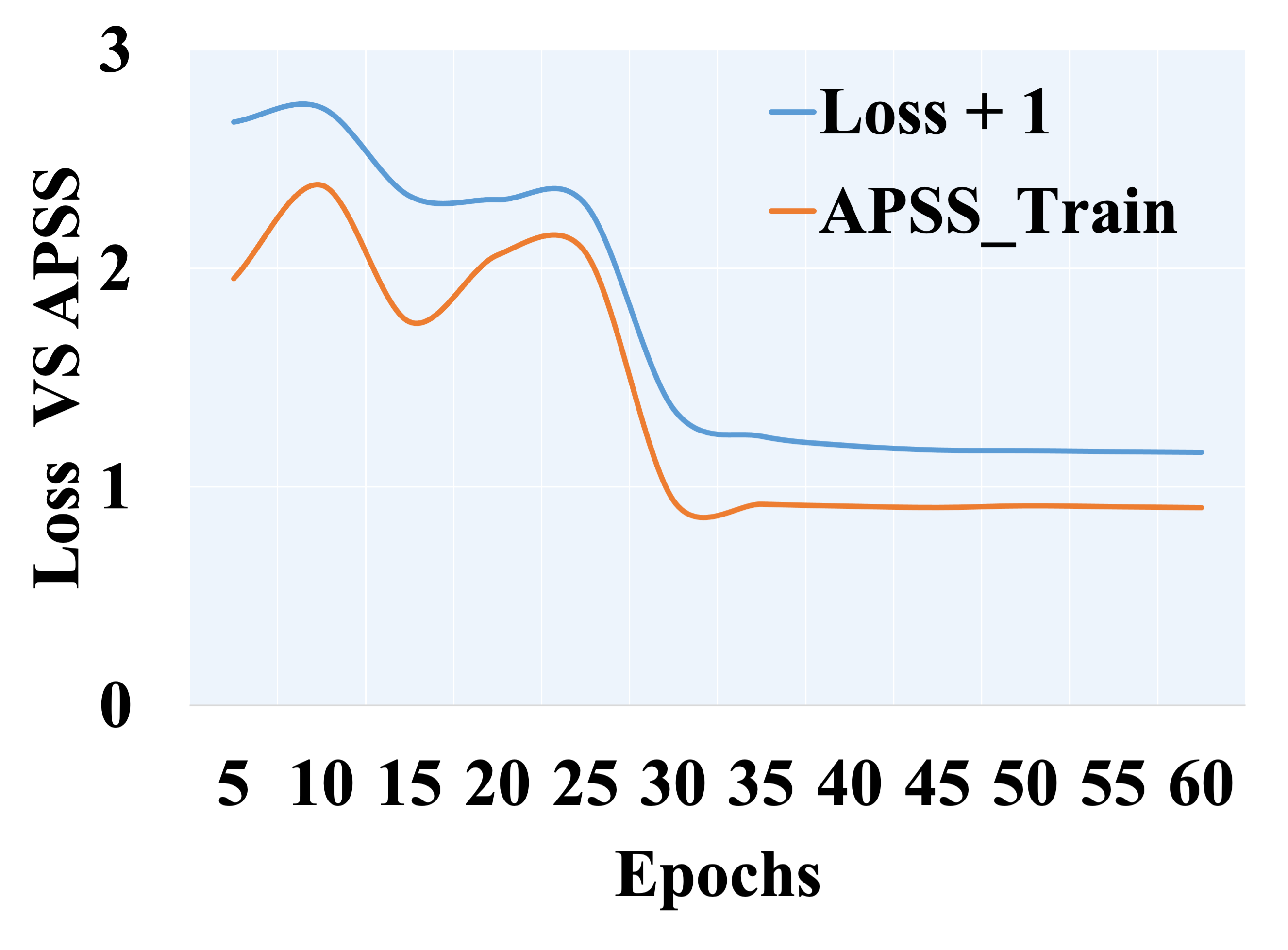}
    \end{minipage} 
    \begin{minipage}[t]{0.32\linewidth}
    \centering
    \textbf{(c)}    APSS Comparison 
     \includegraphics[width=\linewidth]{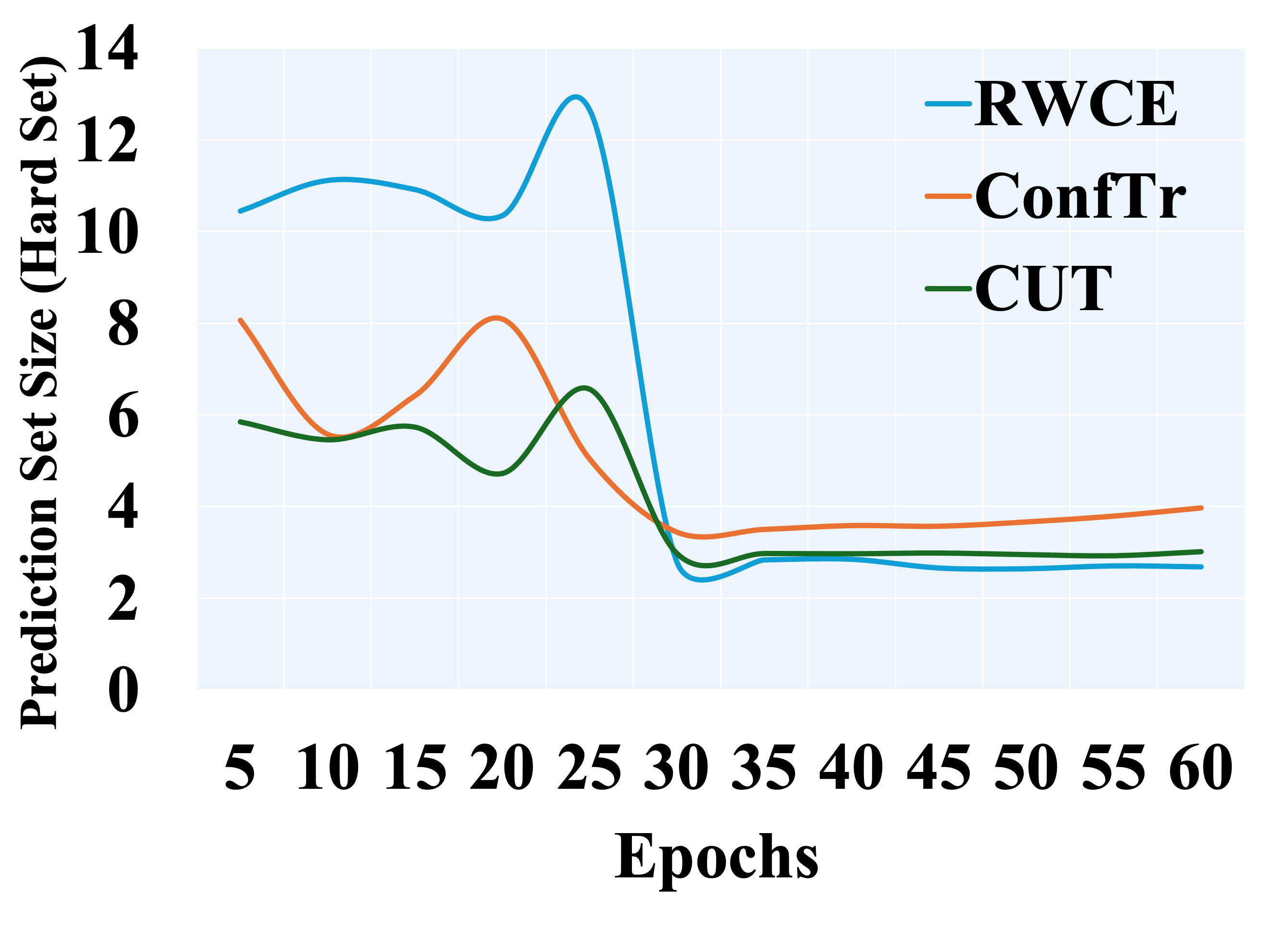}
    \\   
    \includegraphics[width=\linewidth]{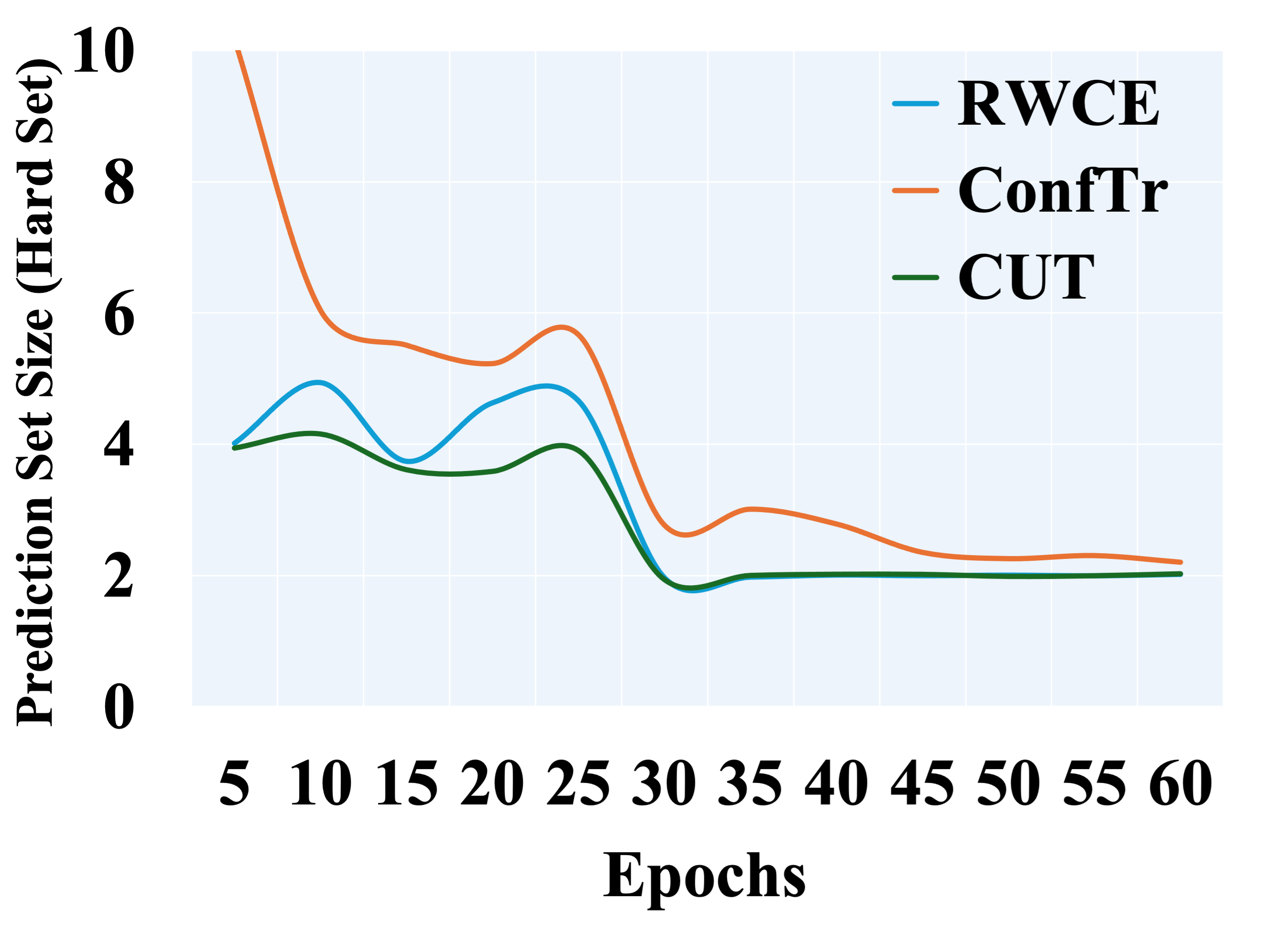}
    \end{minipage} 
    \caption{
    \textbf{Justification experiments using ResNet (Top row) and DenseNet (Bottom row)} on CIFAR-100.
    \textbf{(a)} the training loss convergence of RWCE. The results demonstrate that RWCE converges smoothly and stably on both architectures.  
    \textbf{(b)} the RWCE training loss inflated by $1$ (according to remark~\ref{remark:new_objective_bounds_rank}) with the actual APSS on the validation set. 
    The loss closely upper bounds APSS with a small and stable gap, indicating that RWCE effectively approximates and directly minimizes the true set size objective.  
    \textbf{(c)} the APSS of RWCE, ConfTr, and CUT calibrated by HPS score. RWCE consistently achieves smaller prediction sets on both architectures, confirming its superior efficiency in directly minimizing set size.
    }
    \label{fig:results_overall}
\end{figure*}

\textbf{RWCE generates smaller prediction sets.} Table~\ref{tab:cvg_set_hps_train_main_part} reports the APSS of all methods across three datasets under a fixed coverage level of $\alpha = 0.1$. 
All models are evaluated using both HPS and APS scoring functions during the evaluation phase, with details provided in Appendix~\ref{subsec:train_test_strategies}.
Our method, RWCE, demonstrates consistently strong performance across all datasets and evaluation settings by producing smaller and more efficient prediction sets than existing baselines.
Averaging the relative improvements across all configurations, RWCE achieves a 21.38\% reduction in prediction set size compared to the strongest baseline in each case.
Under HPS evaluation, RWCE consistently achieves the smallest APSS across all dataset--model combinations.  
The gains are even more pronounced under APS evaluation, further reinforcing RWCE’s effectiveness in producing compact and well-calibrated prediction sets---except for a slight degradation on CIFAR-100 with DenseNet, where RWCE still remains competitive.
For completeness, coverage statistics and additional experiments using RAPS and SAPS scoring functions are reported in Appendix \ref{appendix:subsec:additional_exps}.  
Additional experiments on SST-5 with the BERT model are also reported in Appendix \ref{appendix:subsec:additional_exps}.
We further visualize the APSS comparison of the three methods on CIFAR-100 in Figure~\ref{fig:results_overall}~(c), using both ResNet and DenseNet evaluated by the HPS score. The figure clearly demonstrates that RWCE yields significantly smaller prediction sets after each method converges.

\textbf{RWCE Converges Stably During Training.}  
To assess the optimization dynamics of RWCE, we visualize its training loss (according to remark~\ref{remark:new_objective_bounds_rank}) over 60 fine-tuning epochs on CIFAR-100 using both ResNet and DenseNet architectures, as shown in Figure~\ref{fig:results_overall}(a). The training curves exhibit a consistent convergence pattern across both models. Specifically, RWCE experiences some oscillations in the early epochs—particularly visible in ResNet where multiple sharp peaks occur within the first 25 epochs—followed by a sharp decline between epoch 25 and 30. 
After this point, the loss stabilizes rapidly and converges smoothly by around epoch 40.
This loss corresponds to a rank-weighted objective based on the rank of the true label, which we theoretically show to be a tight upper bound on the expected prediction set size. 
The observed convergence behavior thus not only indicates optimization stability, but also reflects the ability of RWCE to directly approximate the set size minimization objective without relying on relaxed approximations of the indicator function.


\textbf{RWCE Directly Minimizes Prediction Set Size.}  
To examine whether RWCE effectively minimizes the true prediction set size, we compare its training loss with the actual APSS on the training set, as shown in Figure~\ref{fig:results_overall}(b). For both ResNet and DenseNet, the training loss inflated with $1$ (according to remark~\ref{remark:new_objective_bounds_rank}) closely upper bounds the APSS throughout training, with the two curves converging to similar values after epoch 40. This tight alignment empirically validates our theoretical claim that the RWCE objective serves as a provable upper bound on the expected set size.
Furthermore, we track the test-time APSS of RWCE, ConfTr, and CUT under HPS calibration during training in Figure~\ref{fig:results_overall}(c). 
While RWCE exhibits some fluctuation in the early stages of training—especially on the ResNet backbone—it rapidly stabilizes and consistently achieves the smallest prediction sets in the later epochs across both model architectures. This highlights its superior efficiency in directly minimizing prediction set size. Although ConfTr converges more smoothly, its final prediction sets remain larger than those of RWCE, which may reflect the limitations of optimizing a soft relaxation of the true objective. Together, these results suggest that RWCE not only aligns well with the evaluation metric but also enables the construction of compact and calibrated prediction sets without relying on soft indicators.

\begin{figure}[!t]
    \centering
    \begin{minipage}[t]{0.49\linewidth}
    \centering
    \textbf{(a)} Alignment Inequality on ResNet
    \end{minipage} 
    \begin{minipage}[t]{0.49\linewidth}
    \centering
    \textbf{(b)} Alignment Inequality on DenseNet
    \end{minipage} 
    \hfill
    \begin{minipage}[t]{0.49\linewidth}
     \centering   
     \includegraphics[width = \linewidth]{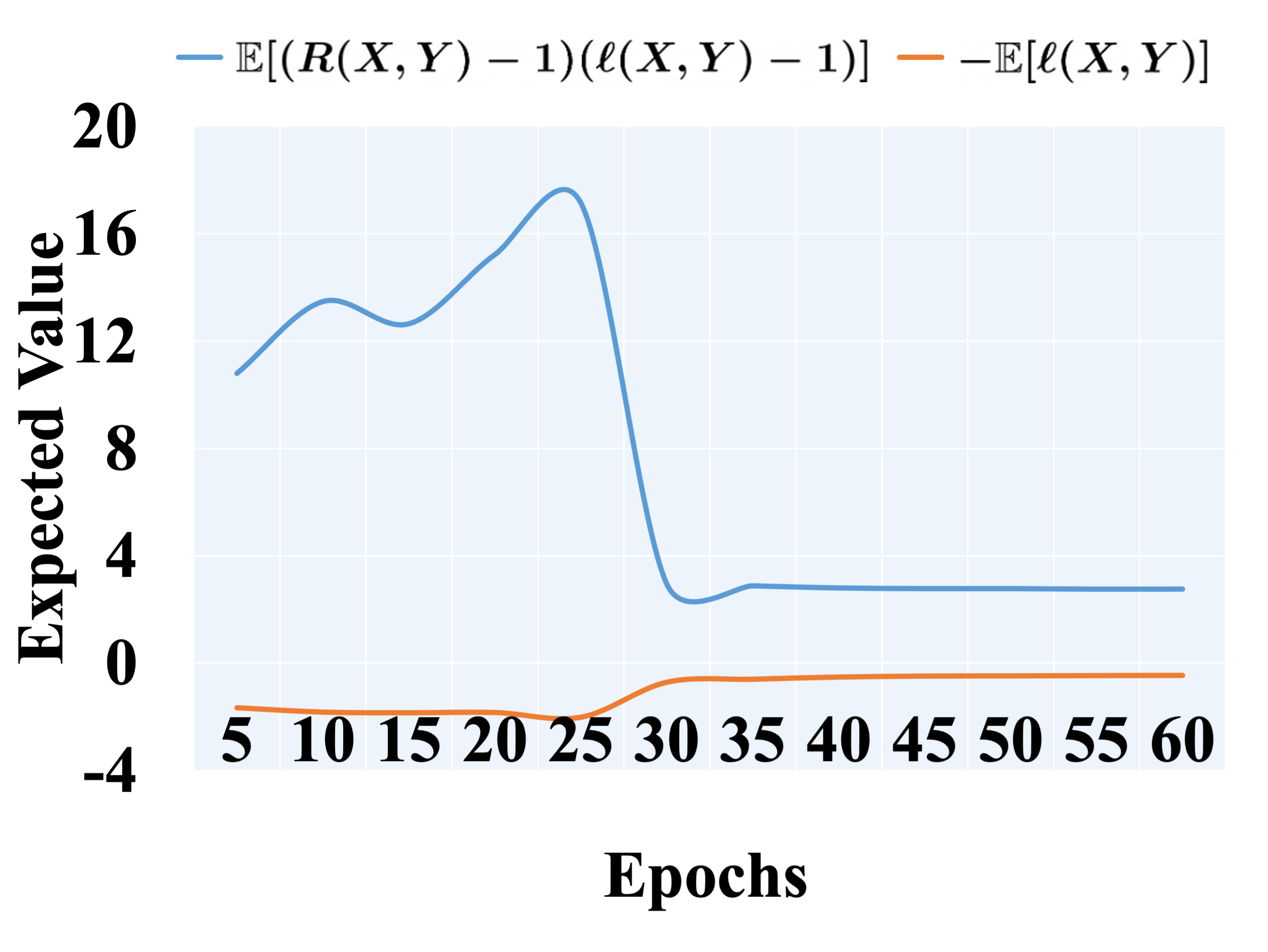}
     \end{minipage}
    \begin{minipage}[t]{0.49\linewidth}
    \centering
    \includegraphics[width=\linewidth]{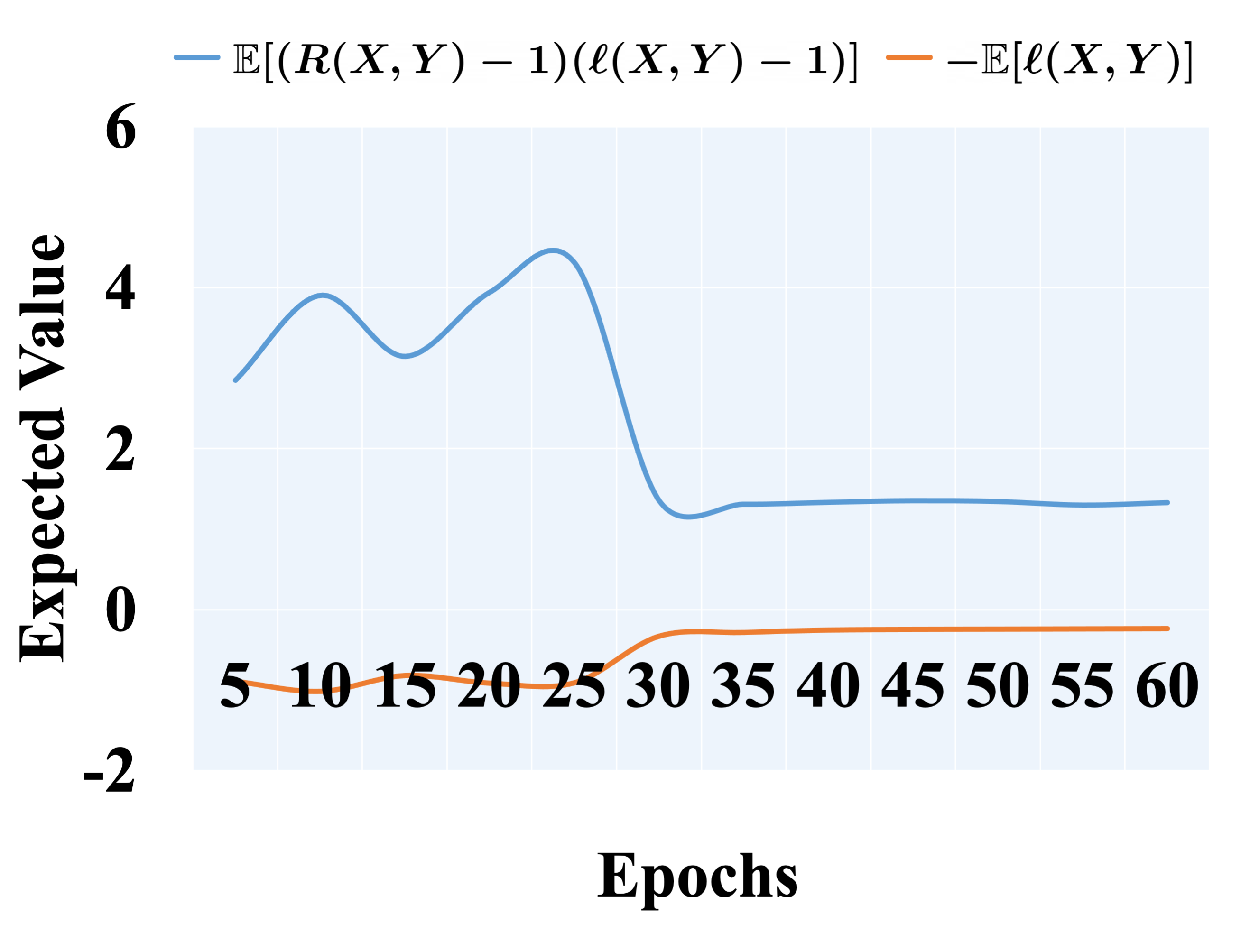}
    \end{minipage}
    \caption{
    \textbf{
    Empirical validation of the alignment assumption in Theorem~\ref{theorem:new_objective_upper_bounds_expected_rank}, which states that $\mathbb{E}[(R(X,Y) - 1)(\ell(X,Y) - 1)] \geq -\mathbb{E}[\ell(X,Y)]$} during training on CIFAR-100 using ResNet (a) and DenseNet (b).  
    We plot the left-hand side (expected rank–loss interaction term, blue) and right-hand side (negative expected cross-entropy loss, orange). 
    In both cases, the blue curve remains significantly above the orange curve throughout training, confirming that the inequality holds in practice.
    }
    \label{fig:rank_ce_alignment}
    \vspace{-4.0ex}
\end{figure}


\textbf{Assumption in Theorem~\ref{theorem:new_objective_upper_bounds_expected_rank} is empirically valid.}
To verify the practical validity of the alignment assumption required by Theorem~\ref{theorem:new_objective_upper_bounds_expected_rank}, we compute and track both sides of the inequality throughout training. 
Specifically, we compare the left-hand side $\mathbb{E}[(R(X,Y) - 1)(\ell(X,Y) - 1)] $ and the right-hand side $-\mathbb{E}[\ell(X,Y)]$. 
Figure~\ref{fig:rank_ce_alignment} shows the two curves on the CIFAR-100 training set using both ResNet and DenseNet architectures across all epochs. 
In both models, we observe that the expected rank–loss interaction term (blue curve, i.e., $\mathbb{E}[(R(X,Y) - 1)(\ell(X,Y) - 1)] $) consistently exceeds the magnitude of the average negative cross-entropy loss (orange curve, i.e., $-\mathbb{E}[\ell(X,Y)]$) throughout training. 
This confirms that the inequality assumed in Theorem~\ref{theorem:new_objective_upper_bounds_expected_rank} holds empirically. These results provide strong support for the theoretical justification of our objective, demonstrating that the rank-weighted loss meaningfully upper bounds the expected rank in real-world training scenarios.

\section{Conclusion}
In this paper, we proposed RWCE, a simple and theoretically grounded conformal training method that avoids the use of surrogate indicator approximations. Instead, RWCE minimizes a rank-weighted cross-entropy loss, where each training sample is reweighted according to the rank of its true label. 
We provide rigorous theoretical analysis showing that this objective tightly upper bounds the expected prediction set size and admits favorable generalization guarantees under mild conditions. 
Empirical results on benchmark datasets confirm the effectiveness of RWCE: it consistently reduces prediction set size compared to prior conformal training methods with an average $21.38\%$ reduction, while preserving valid coverage. 
Our approach demonstrates a principled and practical alternative for conformal prediction, with implications for building adaptive and efficient uncertainty quantification methods.


\section*{Acknowledgements}

The authors gratefully acknowledge the in part support by the USDA-NIFA funded AgAID Institute
award 2021-67021-35344, and the NSF grant CNS-2312125, IIS-2443828, DUE-2519063. 
The views expressed are those of
the authors and do not reflect the official policy or position of the USDA-NIFA and NSF.



\bibliography{aaai2026}

\cleardoublepage

\newpage

\appendix
 
\onecolumn

\section{ Technical Proofs for Main Resuts }
\label{section:appendix:proofs_main_results}

\subsection{ Proof for Theorem \ref{theorem:expected_rank_upper_bounds_expected_size} }
\label{subsection:appendix:proof:theorem:exptected_rank_upper_bounds_expected_size}

\begin{theoreminappendix}
\label{theorem:appendix:expected_rank_upper_bounds_expected_size}
(Theorem \ref{theorem:expected_rank_upper_bounds_expected_size} restated)
\begin{align*}
\E [ | \widehat C(X) | ]
\leq 
\E [ R(X, Y) ]
+ K \bigg( 1 - \alpha + \frac{1}{n+1} \bigg)
.
\end{align*}
\end{theoreminappendix}

\begin{proof}
(of Theorem \ref{theorem:expected_rank_upper_bounds_expected_size})

Proving Theorem \ref{theorem:expected_rank_upper_bounds_expected_size} requires the following four technical lemmas:

\begin{lemma}
\label{lemma:upper_bound_A}
The following inequality holds:
\begin{align}\label{eq:upper_bound_A}
&
\sum_{ y = 1 }^K \indicator [ S(X, y) \leq \widehat Q ] \cdot \indicator[ R(X, Y) \geq R(X, y) ]
\nonumber\\
\leq &
\sum_{ y = 1 }^K \indicator[ R(X, Y) \geq R(X, y) ] \cdot \indicator[ S(X, Y) \leq \widehat Q ] 
+ \sum_{ y = 1 }^K \indicator[ R(X, Y) \geq R(X, y) ] \cdot \indicator[ S(X, y) < S(X, Y) ]
.
\end{align}
\end{lemma}

\begin{lemma}
\label{lemma:upper_bound_B}
If the nonconformity score function $S$ is HPS, APS, or their variants,
then the following inequality holds:
\begin{align}\label{eq:upper_bound_B}
\sum_{ y = 1 }^K \indicator [ S(X, y) \leq \widehat Q ] \cdot \indicator[ R(X, Y) < R(X, y) ]
\leq 
\sum_{ y = 1 }^K \indicator[ R(X, Y) < R(X, y) ] \cdot \indicator[ S(X, Y) \leq \widehat Q ] 
.
\end{align}
\end{lemma}

\begin{lemma}
\label{lemma:useful_inequalities_in_cp_hps_aps}
If the nonconformity score function $S$ is HPS, APS, or their variants,
then we have the following inequalities:
\begin{align}\label{eq:useful_inequalities_in_cp_hps_aps_1}
(1) ~&~
\indicator[ R(X, Y) \geq R(X, y) ] \cdot \indicator[ S(X, y) < S(X, Y) ]
\leq \indicator[ R(X, Y) \geq R(X, y) ]
,
\\
\label{eq:useful_inequalities_in_cp_hps_aps_2}
(2) ~&~
\indicator[ S(X, y) \leq \widehat Q ] \cdot \indicator[ R(X, Y) < R(X, y) ] \cdot \indicator[ S(X, Y) > \widehat Q ] = 0
.
\end{align}
\end{lemma}

\begin{lemma}
\label{lemma:cp_coverage_bounds}
(Immediate results from Theorem 1 in \cite{romano2020classification})
For CP algorithms, regardless which nonconformity score function $S$ is used, 
we have the following lower and upper bounds for the coverage:
\begin{align}\label{eq:cp_coverage_bounds}
1 - \alpha 
\leq 
\P\{ Y \in \widehat C(X) \}
\leq  
1 - \alpha + 1/n
.
\end{align}
\end{lemma}

The proofs for Lemma \ref{lemma:appendix:upper_bound_A}, Lemma \ref{lemma:appendix:upper_bound_B} and Lemma \ref{lemma:useful_inequalities_in_cp_hps_aps} are deferred to Appendix \ref{subsection:appendix:proof:lemma:upper_bound_A}, Appendix \ref{subsection:appendix:proof:lemma:upper_bound_B} and Appendix \ref{subsection:appendix:proof:lemma:useful_inequalities_in_cp_hps_aps}, respectively.
Now we start proving Theorem \ref{theorem:expected_rank_upper_bounds_expected_size} by decomposing the expected size of conformal prediction sets as follows:

\begin{align*}
\E[ | \widehat C(X) | ]
= &
\E\Bigg[ \sum_{ y = 1 }^K \indicator [ S(X, y) \leq \widehat Q ] \Bigg]
=
\E\Bigg[ 
\sum_{ y = 1 }^K \indicator [ S(X, y) \leq \widehat Q ] \cdot 
\bigg( 
\indicator[ R(X, Y) \geq R(X, y) ] + \indicator[ R(X, Y) < R(X, y) ] 
\bigg) 
\Bigg]
\\
= &
\E\Bigg[ 
\underbrace{ 
\sum_{ y = 1 }^K \indicator [ S(X, y) \leq \widehat Q ] \cdot \indicator[ R(X, Y) \geq R(X, y) ]
}_{ \text{(\ref{eq:upper_bound_A}), Lemma \ref{lemma:upper_bound_A}} }
+ \underbrace{ 
\sum_{ y = 1 }^K \indicator [ S(X, y) \leq \widehat Q ] \cdot \indicator[ R(X, Y) < R(X, y) ]
}_{ \text{(\ref{eq:upper_bound_B}), Lemma \ref{lemma:upper_bound_B}} }
\Bigg]
\\
\stackrel{ (a) }{ \leq } &
\E \Bigg[ 
\sum_{ y = 1 }^K \indicator[ R(X, Y) \geq R(X, y) ] \cdot \indicator[ S(X, Y) \leq \widehat Q ] 
+ \sum_{ y = 1 }^K \indicator[ R(X, Y) \geq R(X, y) ] \cdot \indicator[ S(X, y) < S(X, Y) ]
\\
& 
+ \sum_{ y = 1 }^K \indicator[ R(X, Y) < R(X, y) ] \cdot \indicator[ S(X, Y) \leq \widehat Q ] 
\Bigg]
\\
= &
\E \Bigg[ 
\sum_{ y = 1 }^K \bigg( 
\underbrace{ 
\indicator[ R(X, Y) \geq R(X, y) ] + \indicator[ R(X, Y) < R(X, y) ] 
}_{ = 1 }
\bigg) \cdot \indicator[ S(X, Y) \leq \widehat Q ] 
\\
&
+ \sum_{ y = 1 }^K \indicator[ R(X, Y) \geq R(X, y) ] \cdot \indicator[ S(X, y) < S(X, Y) ]
\Bigg]
\\
= &
K \cdot 
\underbrace{ 
\E \Big[ \indicator[ S(X, Y) \leq \widehat Q ] \Big]
}_{ \text{(\ref{eq:cp_coverage_bounds}), Lemma \ref{lemma:cp_coverage_bounds}} }
+ \E \Bigg[ 
\sum_{ y = 1 }^K \underbrace{ 
\indicator[ R(X, Y) \geq R(X, y) ] \cdot \indicator[ S(X, y) < S(X, Y) ]
}_{ \leq \indicator[ R(X, Y) \geq R(X, y) ], \text{ (\ref{eq:useful_inequalities_in_cp_hps_aps_1}), Lemma \ref{lemma:useful_inequalities_in_cp_hps_aps}} }
\Bigg]
\\
\stackrel{ (b) }{ \leq } &
K ( 1 - \alpha + 1 / n )
+ \E \Bigg[ 
\sum_{ y = 1 }^K 
\indicator[ R(X, Y) \geq R(X, y) ]
\Bigg]
=
K ( 1 - \alpha + 1 / n )
+ \E [ R(X, Y) ]
,
\end{align*}
where the inequality $(a)$ is due to Lemma \ref{lemma:upper_bound_A} and Lemma \ref{lemma:upper_bound_B}.

This completes the proof for Theorem \ref{theorem:expected_rank_upper_bounds_expected_size}.
\end{proof}

\subsection{ Proof for Theorem \ref{theorem:new_objective_upper_bounds_expected_rank} }
\label{subsection:appendix:proof:new_objective_upper_bounds_expected_rank}

\begin{theorem}
\label{theorem:new_objective_upper_bounds_expected_rank_appendix}
(Theorem \ref{theorem:new_objective_upper_bounds_expected_rank} restated)
Suppose that $\E[(R(X, Y) - 1)(\ell(X, Y) - 1)] \geq -\E[\ell(X, Y)]$ or $\ell$ is the cross-entropy loss,
then the following inequality holds:
\begin{align*}
\E[R(X, Y)-1]
\leq 
\E[ R(X, Y) \cdot \ell(X, Y)]
=
\calL(f)
.
\end{align*}
\end{theorem}

\begin{proof}


The proof of Theorem \ref{theorem:new_objective_upper_bounds_expected_rank} depends on two different conditions: (1) $\E[(R(X, Y) - 1)(\ell(X, Y) - 1)] \geq -\E[\ell(X, Y)]$; (2) $\ell$ is the cross-entropy loss.

Now we discuss the two cases, respectively.

{\it Case (i): the condition $\E[(R(X, Y) - 1)(\ell(X, Y) - 1)] \geq -\E[\ell(X, Y)]$ holds.}

Then we have:

\begin{align*}
0 &
\leq \E[(R(X, Y) - 1)(\ell(X, Y) - 1)] + \E[\ell(X, Y)]
\\
&
= \E[(R(X, Y) - 1)(\ell(X, Y) - 1) + \ell(X, Y)]
\\
&
= \E[(R(X, Y) \cdot \ell(X, Y) - R(X, Y) - \ell(X, Y) + 1 + \ell(X, Y)]
\\
&
= \E[(R(X, Y) \cdot \ell(X, Y) - R(X, Y) + 1 ]
\\
&
= \E[(R(X, Y) \cdot \ell(X, Y)] - \E[R(X, Y) - 1 ],
\end{align*}
where the first inequality is due to condition that $\E[(R(X, Y) - 1)(\ell(X, Y) - 1)] \geq -\E[\ell(X, Y)]$.

Rearranging the above inequality, we have:
\begin{align*}
\E[R(X, Y) - 1 ] \leq \E[(R(X, Y) \cdot \ell(X, Y)].
\end{align*}

{\it Case (ii): $\ell$ is the cross-entropy loss.}

For given input $(X,Y)$, we define the sorted softmax probabilities for all classes \{$1,\cdots,K$\} such that $1 \ge f_{\theta}(x)_{( 1 )} \geq \cdots \ge f_{\theta}(x)_{( K )}\ge 0$
Thus, we have $\sum^{R(X,Y)}_{l=1} \frac{1}{R(X,Y)} f_{(l)}(X) \geq f_{(R(X,Y))}(X)$ for rank $R(X,Y)$, i.e., the mean of the largest $R(X,Y)$-th softmax scores is larger than the $R(X,Y)$-th score.

Then, we have the following inequality holds for any $(X,Y)$
\begin{align*}
1 
= \sum^K_{l=1} f_{(l)}(X) 
\geq \sum^{R(X,Y)}_{l=1} f_{(l)}(X) 
\geq R(X,Y) \cdot f_{(R(X,Y))}(X)
= R(X,Y) \cdot f_{Y}(X).
\end{align*}

Rearranging the above inequality, we have: $R(X,Y) \cdot f_{Y}(X) \leq 1$ for any $(X,Y)$. Taking the expectation for $(X,Y) \sim \calP$, we have: $\E[R(X,Y)\cdot f_Y(X)] \leq 1$

Recall that: $\ell(X, Y) = -\log f_Y(X)$.
According to the inequality that $- \log u \geq 1 - u$ for $u \in (0,1]$, we have:
\begin{align*}
&
\E[R(X,Y)] -1 
\\
\leq &
\E[R(X,Y)] -\E[R(X,Y)\cdot f_Y(X)] 
\\
= &
\E \Big [ R(X,Y) \big (1 -f_Y(X) \big ) \Big ]
\\
\leq &
\E [ R(X,Y) (- \log f_Y(X) ) ]
\\
= &
\E [ R(X,Y) \ell(X, Y)].
\end{align*}

Rearranging the above inequality, we have:
\begin{align*}
\E[R(X,Y)-1] \leq \E [ R(X,Y) \ell(X, Y)].
\end{align*}





\end{proof}

\section{ Proofs for Technical Lemmas }
\label{section:appendix:proofs_for_lemmas}

\subsection{ Proof for Lemma \ref{lemma:upper_bound_A} }
\label{subsection:appendix:proof:lemma:upper_bound_A}

\begin{lemmainappendix}
\label{lemma:appendix:upper_bound_A}
(Lemma \ref{lemma:upper_bound_A} restated)
The following inequality holds:
\begin{align*}
&
\sum_{ y = 1 }^K \indicator [ S(X, y) \leq \widehat Q ] \cdot \indicator[ R(X, Y) \geq R(X, y) ]
\nonumber\\
\leq &
\sum_{ y = 1 }^K \indicator[ R(X, Y) \geq R(X, y) ] \cdot \indicator[ S(X, Y) \leq \widehat Q ] 
+ \sum_{ y = 1 }^K \indicator[ R(X, Y) \geq R(X, y) ] \cdot \indicator[ S(X, y) < S(X, Y) ]
.
\end{align*}
\end{lemmainappendix}

\begin{proof}
(of Lemma \ref{lemma:upper_bound_A})

\begin{align*}
&
\sum_{ y = 1 }^K \indicator [ S(X, y) \leq \widehat Q ] \cdot \indicator[ R(X, Y) \geq R(X, y) ]
\\
= &
\sum_{ y = 1 }^K \indicator [ S(X, y) \leq \widehat Q ] \cdot \indicator[ R(X, Y) \geq R(X, y) ] \cdot \bigg( \indicator[ S(X, Y) \leq \widehat Q ] + \indicator[ S(X, Y) > \widehat Q ] \bigg)
\\
= &
\sum_{ y = 1 }^K \underbrace{ \indicator [ S(X, y) \leq \widehat Q ] }_{ \leq 1 } \cdot \indicator[ R(X, Y) \geq R(X, y) ] \cdot \indicator[ S(X, Y) \leq \widehat Q ] 
\\
&
+ \sum_{ y = 1 }^K \indicator [ S(X, y) \leq \widehat Q ] \cdot \indicator[ R(X, Y) \geq R(X, y) ] \cdot \indicator[ S(X, Y) > \widehat Q ] 
\\
\stackrel{ (a) }{ \leq } &
\sum_{ y = 1 }^K 1 \cdot \indicator[ R(X, Y) \geq R(X, y) ] \cdot \indicator[ S(X, Y) \leq \widehat Q ] 
+ \sum_{ y = 1 }^K \indicator[ R(X, Y) \geq R(X, y) ] \cdot \indicator[ S(X, y) \leq \widehat Q < S(X, Y) ]
\\
\stackrel{ (b) }{ \leq } &
\sum_{ y = 1 }^K \indicator[ R(X, Y) \geq R(X, y) ] \cdot \indicator[ S(X, Y) \leq \widehat Q ] 
+ \sum_{ y = 1 }^K \indicator[ R(X, Y) \geq R(X, y) ] \cdot \indicator[ S(X, y) < S(X, Y) ]
,
\end{align*}
where the inequality $(a)$ is due to $\indicator[\cdot] \leq 1$, and
the inequality $(b)$ is due to $\indicator[x \leq y < z] \leq \indicator[x < z]$.
This completes the proof for Lemma \ref{lemma:upper_bound_A}.
\end{proof}

\subsection{ Proof for Lemma \ref{lemma:upper_bound_B} }
\label{subsection:appendix:proof:lemma:upper_bound_B}

\begin{lemmainappendix}
\label{lemma:appendix:upper_bound_B}
(Lemma \ref{lemma:upper_bound_B} restated)
If the nonconformity score function $S$ is HPS, APS, or their variants,
then the following inequality holds:
\begin{align*}
\sum_{ y = 1 }^K \indicator [ S(X, y) \leq \widehat Q ] \cdot \indicator[ R(X, Y) < R(X, y) ]
\leq 
\sum_{ y = 1 }^K \indicator[ R(X, Y) < R(X, y) ] \cdot \indicator[ S(X, Y) \leq \widehat Q ] 
.
\end{align*}
\end{lemmainappendix}

\begin{proof}
(of Lemma \ref{lemma:upper_bound_B})

\begin{align*}
&
\sum_{ y = 1 }^K \indicator [ S(X, y) \leq \widehat Q ] \cdot \indicator[ R(X, Y) < R(X, y) ]
\\
= &
\sum_{ y = 1 }^K \indicator [ S(X, y) \leq \widehat Q ] \cdot \indicator[ R(X, Y) < R(X, y) ] \cdot \bigg( \indicator[ S(X, Y) \leq \widehat Q ] + \indicator[ S(X, Y) > \widehat Q ] \bigg)
\\
= &
\sum_{ y = 1 }^K \underbrace{ \indicator [ S(X, y) \leq \widehat Q ] }_{ \leq 1 } \cdot \indicator[ R(X, Y) < R(X, y) ] \cdot \indicator[ S(X, Y) \leq \widehat Q ] 
\\
&
+ \sum_{ y = 1 }^K \underbrace{ 
\indicator [ S(X, y) \leq \widehat Q ] \cdot \indicator[ R(X, Y) < R(X, y) ] \cdot \indicator[ S(X, Y) > \widehat Q ] 
}_{ = 0, \text{ Lemma \ref{lemma:useful_inequalities_in_cp_hps_aps}} }
\\
\stackrel{ (a) }{ \leq } &
\sum_{ y = 1 }^K \indicator[ R(X, Y) < R(X, y) ] \cdot \indicator[ S(X, Y) \leq \widehat Q ] 
,
\end{align*}
where the inequality $(a)$ is due to $\indicator[\cdot] \leq 1$, and 
$\indicator[ S(X, y) \leq \widehat Q ] \cdot \indicator[ R(X, Y) < R(X, y) ] \cdot \indicator[ S(X, Y) \leq \widehat Q ] = 0$ if the nonconformity score function $S$ is HPS, APS, or their variants, according to (\ref{eq:useful_inequalities_in_cp_hps_aps_2}), Lemma \ref{lemma:useful_inequalities_in_cp_hps_aps}.

This completes the proof for Lemma \ref{lemma:upper_bound_B}.
\end{proof}

\subsection{ Proof for Lemma \ref{lemma:useful_inequalities_in_cp_hps_aps} }
\label{subsection:appendix:proof:lemma:useful_inequalities_in_cp_hps_aps}

\begin{lemmainappendix}
\label{lemma:appendix:useful_inequalities_in_cp_hps_aps}
(Lemma \ref{lemma:useful_inequalities_in_cp_hps_aps} restated)
If the nonconformity score function $S$ is HPS, APS, or their variants,
then we have the following inequalities:
\begin{align*}
(1) ~&~
\indicator[ R(X, Y) \geq R(X, y) ] \cdot \indicator[ S(X, y) < S(X, Y) ]
\leq \indicator[ R(X, Y) \geq R(X, y) ]
,
\\
(2) ~&~
\indicator[ S(X, y) \leq \widehat Q ] \cdot \indicator[ R(X, Y) < R(X, y) ] \cdot \indicator[ S(X, Y) > \widehat Q ] = 0
.
\end{align*}
\end{lemmainappendix}

\begin{proof}
(of Lemma \ref{lemma:useful_inequalities_in_cp_hps_aps})

Before proving Lemma \ref{lemma:useful_inequalities_in_cp_hps_aps}, we first the the following technical lemma:
\begin{lemma}
\label{lemma:tech_lemma_1_for_lemma_3}
\begin{align*}
\indicator\bigl[f(X)_y > f(X)_Y\bigr]
\;=\;
\indicator\bigl[R(X,Y) > R(X,y)\bigr].
\end{align*}
\end{lemma}

The proof of Lemma \ref{lemma:tech_lemma_1_for_lemma_3} is deferred at the end of proof of Lemma \ref{lemma:useful_inequalities_in_cp_hps_aps}. Now we begin to prove Lemma \ref{lemma:useful_inequalities_in_cp_hps_aps}.
This proof consists of two parts, (1) Equation (\ref{eq:useful_inequalities_in_cp_hps_aps_1}): $\indicator[ R(X, Y) \geq R(X, y) ] \cdot \indicator[ S(X, y) < S(X, Y) ]
= \indicator[ R(X, Y) \geq R(X, y) ]$ ;
and (2) Equation (\ref{eq:useful_inequalities_in_cp_hps_aps_2}): $\indicator[ S(X, y) \leq \widehat Q ] \cdot \indicator[ R(X, Y) < R(X, y) ] \cdot \indicator[ S(X, Y) > \widehat Q ] = 0$.

{\bf Proof of Equation (\ref{eq:useful_inequalities_in_cp_hps_aps_1}):}

Recall the definitions of HPS and APS scoring functions:
\begin{align*}
S^{\HPS}(X, Y) 
&
=  
1 - f(X)_Y,
\\
S^{\APS}(X, Y) 
&
=  
\sum_{l=1}^{R(X, Y)-1} f(X)_{(l)} + U \cdot f(X)_{(R(X, Y))} 
,
\end{align*}
where $U \in [0,1]$ is a uniform random variable to break ties.

Then we discuss the two cases where HPS and APS are used as the nonconformity score function $S$, respectively.

\medskip

{\it Case (i): HPS is used as the nonconformity score function $S$.}

\begin{align*}
&
\indicator[ R(X, Y) \geq R(X, y) ] \cdot \indicator[ S(X, y) < S(X, Y) ]
\nonumber \\
=
&
\indicator[ R(X, Y) \geq R(X, y) ] \cdot \indicator[ 1 - f(X)_y < 1 - f(X)_Y ]
\nonumber \\
=
&
\indicator[ R(X, Y) \geq R(X, y) ] \cdot \indicator[ \underbrace{f(X)_y > f(X)_Y}_{ = R(X, Y) > R(X, y), \text{ Lemma \ref{lemma:tech_lemma_1_for_lemma_3}} } ]
\nonumber \\
=
&
\indicator[ R(X, Y) \geq R(X, y) ] \cdot \indicator[ R(X, Y) > R(X, y) ]
\nonumber \\
=
&
\indicator[ \{ R(X, Y) \geq R(X, y) \} \cap \{ R(X, Y) > R(X, y) \} ] 
\nonumber \\
=
&
\indicator[ R(X, Y) > R(X, y) ], 
\end{align*}
where the first equality is due to the definition of HPS score, the second equality is due to the result of rearranging $1 - f(X)_y < 1 - f(X)_Y$ to $f(X)_y > f(X)_Y$, the third equality is due to Lemma \ref{lemma:tech_lemma_1_for_lemma_3}, and the fourth equality is due to $\indicator[A] \cdot \indicator[B] = \indicator[ A \cap B] $.

\medskip

Define two sets $\{R(X, Y) > R(X, y) \} $ and $ \{ R(X, Y) \geq R(X, y) \}$. Due to $\{R(X, Y) > R(X, y) \} \subseteq \{ R(X, Y) \geq R(X, y) \}$,
the following inequality holds: 
\begin{align*}
\indicator[ R(X, Y) > R(X, y) ] \leq \indicator[ R(X, Y) \geq R(X, y) ]. 
\end{align*}

Therefore, we have:
\begin{align}
\label{eq:lemma_3_hps}
\indicator[ R(X, Y) \geq R(X, y) ] \cdot \indicator[ S(X, y)< S(X, Y) ] 
= 
\indicator[ R(X, Y) > R(X, y) ]
\leq \indicator[ R(X, Y) \geq R(X, y) ]. 
\end{align}

{\it Case (ii): APS is used as the nonconformity score function $S$.}

  Recall that
  \[
    R(X,y)
    \;=\;
    \sum_{\ell\in\calY}
      \indicator\bigl[f(X)_\ell \;\ge\;f(X)_y\bigr],
  \]
  so \(R(X,y)\) is exactly the position of \(y\) in the sorted list
  \(\bigl\{\,f(X)_\ell : \ell\in\calY\bigr\}\) (sorted in descending order).  Likewise,
  \[
    R(X,Y)
    \;=\;
    \sum_{\ell\in\calY}
      \indicator\bigl[f(X)_\ell \;\ge\;f(X)_Y\bigr].
  \]
  
  We now show that, for any fixed \(U\in[0,1]\),
  \[
    R(X,y)\;<\;R(X,Y)
    \quad\Longleftrightarrow\quad
    S^{\APS}(X,y)\;<\;S^{\APS}(X,Y),
  \]
  and if \(R(X,y)=R(X,Y)\) then \(S^{\APS}(X,y)=S(X,Y)\).  From this, it follows immediately that
  \[
    \indicator\bigl[R(X,Y)\ge R(X,y)\bigr]\;
    \indicator\bigl[S(X,y)<S(X,Y)\bigr]
    \;=\;
    \indicator\bigl[R(X,Y)>R(X,y)\bigr].
  \]

  \medskip

  \textbf{Condition 1:} \(R(X,y)<R(X,Y)\).  
  
According to Lemma \ref{lemma:tech_lemma_1_for_lemma_3},
  \[
    R(X,y)<R(X,Y)
    \;\Longrightarrow\;
    f(X)_{(\,R(X,y)\,)}
    \;>\;
    f(X)_{(\,R(X,Y)\,)}.
  \]
  
  Therefore,
  \begin{align*}
    S(X,y)
    \;=\;
    \sum_{\ell=1}^{\,R(X,y)-1\,} f(X)_{(\ell)}
    \;+\;
    U\,f(X)_{(\,R(X,y)\,)}
    &<
    \sum_{\ell=1}^{\,R(X,Y)-1\,} f(X)_{(\ell)}
    \;+\;
    U\,f(X)_{(\,R(X,Y)\,)}
    \;=\;
    S(X,Y),
  \end{align*}
  regardless of the particular value of \(U\).  Hence
  \[
    R(X,y)<R(X,Y)
    \quad\Longrightarrow\quad
    S(X,y)<S(X,Y).
  \]

  \medskip

  \textbf{Condition 2:} \(R(X,y)>R(X,Y)\).  
  
  A completely symmetric argument shows
  \[
    R(X,y)>R(X,Y)
    \quad\Longrightarrow\quad
    S(X,y)>S(X,Y),
  \]
  again for any \(U\in[0,1]\).

  \medskip

  \textbf{Condition 3:} \(R(X,y)=R(X,Y)\).  
  
  In this situation, both sums
  \(\sum_{\ell=1}^{R-1}f(X)_{(\ell)}\) are identical, and both have the same final term \(U\,f(X)_{(R)}\).  Concretely,
  \[
    S(X,y)
    =
    \sum_{\ell=1}^{\,R(X,y)-1\,} f(X)_{(\ell)}
    \;+\;
    U\,f(X)_{(\,R(X,y)\,)}
    =
    \sum_{\ell=1}^{\,R(X,Y)-1\,} f(X)_{(\ell)}
    \;+\;
    U\,f(X)_{(\,R(X,Y)\,)}
    =
    S(X,Y).
  \]
  

  \medskip

  Combining all three conditions shows:
  \begin{itemize}
    \item If \(R(X,y)<R(X,Y)\), then \(S(X,y)<S(X,Y)\).
    \item If \(R(X,y)=R(X,Y)\), then \(S(X,y)=S(X,Y)\).
    \item If \(R(X,y)>R(X,Y)\), then \(S(X,y)>S(X,Y)\).
  \end{itemize}

  Equivalently,
  \[
    R(X,y)<R(X,Y)
    \;\Longleftrightarrow\;
    S(X,y)<S(X,Y),
    \quad\text{and}\quad
    R(X,y)=R(X,Y)
    \;\Longrightarrow\;
    S(X,y)=S(X,Y).
  \]

Therefore, we have:
\begin{align}
\label{eq:lemma_3_aps}
\indicator\bigl[R(X,Y)\ge R(X,y)\bigr]
    \;\cdot\;
    \indicator\bigl[S(X,y)<S(X,Y)\bigr]
    \;=\;
    \indicator\bigl[R(X,Y)>R(X,y)\bigr]
    \; \leq \;
    \indicator\bigl[R(X,Y) \geq R(X,y)\bigr],
\end{align}
where the last inequality is due to $\{R(X, Y) > R(X, y) \} \subseteq \{ R(X, Y) \geq R(X, y) \}$. 

\medskip

Combining Equation (\ref{eq:lemma_3_hps}) for HPS score and (\ref{eq:lemma_3_aps}) for APS score, we finished the proof of Equation (\ref{eq:useful_inequalities_in_cp_hps_aps_1}) when nonconformity score function $S$ is HPS, APS, or their variants:
\begin{align*}
\indicator\bigl[R(X,Y)\ge R(X,y)\bigr]
    \;\cdot\;
    \indicator\bigl[S(X,y)<S(X,Y)\bigr]
    \; \leq \;
    \indicator\bigl[R(X,Y) \geq R(X,y)\bigr]. 
\end{align*}

\medskip

{\bf Proof of Equation (\ref{eq:useful_inequalities_in_cp_hps_aps_2}):}

From Equation (\ref{eq:lemma_3_hps}) for HPS score and (\ref{eq:lemma_3_aps}) for APS score, we have: 

\begin{align*}
\indicator\bigl[R(X,Y)\ge R(X,y)\bigr]
    \;\cdot\;
    \indicator\bigl[S(X,y)<S(X,Y)\bigr]
    \;=\;
    \indicator\bigl[R(X,Y)>R(X,y)\bigr].
\end{align*}

A completely symmetric argument shows: 
\begin{align*}
\indicator\bigl[R(X,Y)\le R(X,y)\bigr]
    \;\cdot\;
    \indicator\bigl[S(X,y)>S(X,Y)\bigr]
    \;=\;
    \indicator\bigl[R(X,Y)<R(X,y)\bigr].
\end{align*}

\medskip

Therefore, we have:
\begin{align*}
&
\indicator[ S(X, y) \leq \widehat Q ] \cdot \indicator[ R(X, Y) < R(X, y) ] \cdot \indicator[ S(X, Y) > \widehat Q ] 
\\
=
&
\indicator[ \{ S(X, y) \leq \widehat Q \} \cap \{ S(X, Y) > \widehat Q\} ] \cdot \indicator[ R(X, Y) < R(X, y) ] 
\\
=
&
\indicator[ S(X, y) \leq \widehat Q < S(X, Y) ] \cdot \indicator[ R(X, Y) < R(X, y) ] 
\\
=
&
\indicator[ \underbrace{S(X, y) \leq \widehat Q < S(X, Y)}_{A_1} ] \cdot \indicator\bigl[R(X,Y)\le R(X,y)\bigr] \cdot
\indicator\bigl[\underbrace{S(X,y)>S(X,Y)}_{A2}\bigr].
\end{align*}
where we define two events $A_1 = S(X, y) \leq \widehat Q < S(X, Y)$ and $A_2 = S(X,y)>S(X,Y)$, and the two events are mutually exclusive.

That is, the following equality holds: 
\begin{align*}
\indicator[S(X, y) \leq \widehat Q < S(X, Y)] \cdot \indicator\bigl[R(X,Y)\le R(X,y)\bigr] \cdot
\indicator\bigl[ S(X,y)>S(X,Y)\bigr]
= 0.
\end{align*}

Therefore, we have:
\begin{align*}
\indicator[ S(X, y) \leq \widehat Q ] \cdot \indicator[ R(X, Y) < R(X, y) ] \cdot \indicator[ S(X, Y) > \widehat Q ] 
= 0.
\end{align*}
\end{proof}

\begin{proof}
(of Lemma \ref{lemma:tech_lemma_1_for_lemma_3})

Recall definition of $R(X, y) = \sum_{l \in \calY} \indicator[ f(X)_l \geq f(X)_y ]$, now we prove that $\indicator[ f(X)_y > f(X)_Y ] = \indicator[ R(X, Y) > R(X, y) ]$:

First, we will prove that:

\noindent\textbf{(1) If $f(X)_y > f(X)_Y$, then $R(X,Y) > R(X,y)$.}

Since \(f(X)_y > f(X)_Y\), any label \(l\) satisfying $f(X)_l \;\ge\; f(X)_y$ must also satisfy $f(X)_l \;\ge\; f(X)_Y$,
because \(f(X)_y\) is strictly larger than \(f(X)_Y\).  

Hence
\begin{align*}
\bigl\{\,l : f(X)_l \ge f(X)_y \bigr\}
    \;\subsetneq\;
    \bigl\{\,l : f(X)_l \ge f(X)_Y \bigr\}.
\end{align*}

Counting the elements of these sets gives
    \[
    R(X,y)
    \;=\;
    \bigl|\{\,l : f(X)_l \ge f(X)_y\}\bigr|
    \;<\;
    \bigl|\{\,l : f(X)_l \ge f(X)_Y\}\bigr|
    \;=\;
    R(X,Y).
    \]
    
That is, \(R(X,Y) > R(X,y)\).  
Therefore, 
\begin{align*}
f(X)_y > f(X)_Y \Longrightarrow R(X,Y) > R(X,y).
\end{align*}


\noindent\textbf{(2) If $R(X,Y) > R(X,y)$, then $f(X)_y > f(X)_Y$.}

Suppose, for the sake of contradiction, that \(f(X)_y \le f(X)_Y\).  There are two subcases:

\emph{Case A:} \(f(X)_y < f(X)_Y\).  In that case, whenever 
    \[
    f(X)_l \;\ge\; f(X)_Y,
    \]
    it follows that
    \[
    f(X)_l \;\ge\; f(X)_y,
    \]
    since \(f(X)_Y\) exceeds \(f(X)_y\).  Thus
    \[
    \bigl\{\,l : f(X)_l \ge f(X)_Y \bigr\}
    \;\subseteq\;
    \bigl\{\,l : f(X)_l \ge f(X)_y \bigr\},
    \]
    and counting yields
    \[
    R(X,Y)
    \;=\;
    \bigl|\{\,l : f(X)_l \ge f(X)_Y\}\bigr|
    \;\le\;
    \bigl|\{\,l : f(X)_l \ge f(X)_y\}\bigr|
    \;=\;
    R(X,y).
    \]
    Hence \(R(X,Y) \le R(X,y)\), contradicting the assumption \(R(X,Y) > R(X,y)\).

\emph{Case B:} \(f(X)_y = f(X)_Y\).  Then
    \[
    \bigl\{\,l : f(X)_l \ge f(X)_y \bigr\}
    \;=\;
    \bigl\{\,l : f(X)_l \ge f(X)_Y \bigr\},
    \]
    
so \(R(X,y) = R(X,Y)\), again contradicting \(R(X,Y) > R(X,y)\).

In either subcase, the assumption \(f(X)_y \le f(X)_Y\) leads to 
    \(R(X,Y) \le R(X,y)\), which contradicts \(R(X,Y) > R(X,y)\).  

Therefore it must be that \(f(X)_y > f(X)_Y\).  Equivalently,
\begin{align*}
R(X,Y) > R(X,y) \implies f(X)_y > f(X)_Y.
\end{align*}

Combining (1) and (2) shows
\[
f(X)_y > f(X)_Y 
\quad\Longleftrightarrow\quad
R(X,Y) > R(X,y),
\]
and hence
\[
\indicator\bigl[f(X)_y > f(X)_Y\bigr]
\;=\;
\indicator\bigl[R(X,Y) > R(X,y)\bigr].
\]

\end{proof}

\section{Additional background details}
\label{appendix:section:additional_details}

\noindent \textbf{nonconformity scoring functions.} The homogeneous prediction sets (HPS) \cite{sadinle2019least} scoring function is
defined as follows:
\begin{align}
\label{eq:HPS}
S_f^{\HPS}(X, Y) = 1 - f_{\theta}(X)_Y.
\end{align}

\cite{romano2020classification} has proposed another conformity scoring function, Adaptive Prediction Sets (APS). 
APS scoring function is based on sorted probabilities.
For a given input $X$, we sort the softmax probabilities for all classes \{$1,\cdots,K$\} such that $1 \ge f_{\theta}(x)_{( 1 )} \geq \cdots \ge f_{\theta}(x)_{( K )}\ge 0$, and compute the cumulative confidence as follows:

\begin{align}
\label{eq:APS}
S_f^{\APS}(X, Y) = \sum^{r_f(X, Y)-1}_{l=1} f_{\theta}(X)_{( l )} + U \cdot f_{\theta}(X)_{( r_f(X, Y) )},
\end{align}
where $U \in [0,1]$ is a uniform random variable to break ties.

To reduce the probability of including unnecessary labels (i.e., labels with high ranks) and thus improve the predictive efficiency, \cite{angelopoulos2021uncertainty} has proposed Regularized Adaptive Prediction Sets (RAPS) scoring function. 
The RAPS score is computed as follows:
\begin{align}
\label{eq:RAPS}
S_f^{\RAPS}(X, Y) = 
\sum^{r_f(X, Y)-1}_{l=1} f_{\theta}(X)_{( l )} + U \cdot f_{\theta}(X)_{( r_f(X, Y) )} + \lambda_{\RAPS} (r_f(X, Y) - k_{\reg})^+,
\end{align}
where $\lambda_{\RAPS}$ and $ k_{\reg}$ are two hyper-parameters.

To reduce the dependence of the conformal score on potentially miscalibrated softmax probabilities while still leveraging label-ranking information, \cite{huang2024conformal} proposed the Sorted Adaptive Prediction Sets (SAPS) scoring function. 
The key idea of SAPS is to discard all probability values except for the maximum softmax probability and to encode the remaining uncertainty purely through the rank of the true label.
The SAPS score is computed as follows:
\begin{align}
\label{eq:SAPS}
S_f^{\SAPS}(X, Y) =
\begin{cases}
U \cdot f_{\theta}(X)_{(1)}, & \text{if } r_f(X, Y) = 1, \\
f{\theta}(X)_{(1)} + \lambda_{\SAPS} \bigl(r_f(X, Y) - 2 + U\bigr) , & \text{otherwise},
\end{cases}
\end{align}
where $f_{\theta}(X)_{(1)} = \max_{c} f_{\theta}(X)_c$ denotes the maximum softmax probability, and $\lambda{\SAPS} > 0$ is a hyper-parameter that controls the strength of the rank-based penalty.

\noindent \textbf{Objective function for conformal training methods.} ConfTr \cite{stutz2021learning} estimates a soft measure of the prediction set size, defined as follows:
\begin{align}
\label{eq:ps_size_approximate_conftr}
\widehat \calL_{\ConfTr}(f) 
= &
\E_{\widehat q_f \sim \calQ_f} \Bigg [ \frac{1}{n} \sum_{i=1}^n \sum_{y \in \calY } \tilde \indicator \big [ S_f(X_i, y) \leq \widehat q_f \big ] \Bigg ]
,
\end{align}
where $\tilde \indicator[\cdot]$ is a smoothed estimator for the indicator function $\indicator[\cdot]$ and defined by a Sigmoid function \cite{stutz2021learning}, 
i.e.,
$\tilde \indicator[ S \leq q ] = 1/(1+\exp(-(q-S)/\tau_\Sigmoid))$ with a tunable temperature parameter $\tau_\Sigmoid$.


CUT \cite{einbinder2022training} measures the maximum deviation from the uniformity of conformity scores, defined as follows:
\begin{align}
\label{eq:ps_size_approximate_cut}
\widehat \calL_{\UA}(f) 
= &
\E_{\widehat q_f \sim \calQ_f} \Big [ \sup_{\alpha \in [0,1]} \big |  (1-\alpha) - \widehat q_f(\alpha) \big | \Big ]
,
\end{align}
where $\widehat q_f(\alpha)$ is the empirical batch-level quantile in $\calB$ of input $\alpha \in [0,1]$.

\section{Additional Experimental Setup Details}
\label{subsec:train_test_strategies}

\noindent \textbf{Dataset and Split.}
We utilize three widely used benchmark datasets CIFAR-100 \cite{krizhevsky2009learning}, Caltech-101 \cite{FeiFei2004LearningGV},
and iNaturalist \cite{van2018iNaturalist}. 
The original test sets are divided into calibration and test subsets to enable more reliable evaluation. 
Table \ref{tab:Data_stat} presents a summary of key dataset statistics, which are elaborated upon in the following sections.  

\vspace{-1ex}

\begin{table*}[!ht]
\centering
\caption{Description of the data sets are given in the table. $^{*}$The number of classes in the iNaturalist data set depends on the taxonomy level (e.g., species, genus, family). We employ "Fungi" species which has 341 different categories.}
\label{tab:Data_stat}
\begin{tabular}{|p{2.8cm}|p{2.3cm}|p{2.3cm}|p{2.3cm}|p{2.3cm}|p{2.3cm}|}
\hline
Data            & Number of Classes    & Number of Training Data      & Number of Validation Data     & Number of Calibration Data      & Number of Test Data   \\ \hline
CIFAR-100        & 100                  & 45000                        & 5000                          & 3000                            & 7000                  \\ \hline
Caltech-101     & 101                  & 4310                         & 1256                          & 1111                            & 2000                  \\ \hline
iNaturalist     & 341*                 & 15345                        & 1705                          & 1410                            & 2000                  \\ \hline
\end{tabular}
\end{table*}

\noindent \textbf{Hyperparameters for training.}
We treat the dataset, base model, batch size, number of training epochs, learning rate, learning schedule, momentum, gamma, weight decay, and temperature as hyperparameters.
Momentum and weight decay help stabilize optimization, gamma controls learning rate decay, and temperature adjusts the smoothness of softmax scores for nonconformity calculation.
We use a gradual, performance-driven hyperparameter tuning strategy tailored to each model--dataset pair. 
Specifically, we consider the learning rate \( \eta \in [0.0001, 0.05] \), 
weight decay \( \lambda \in [0.0001, 0.002] \), and temperature \( T \in [1.0, 2.5] \), 
and evaluate multiple intermediate values based on validation performance 
(e.g., \( \eta = 0.0002, 0.018, 0.025 \)). 
Learning rate schedules are selected from a small set of commonly used decay steps: 
\(\{[3], [5], [18, 30], [25, 35], [25, 40]\}\). 
The number of training epochs is fixed at either 40 or 60, and the batch size is set to either 64 or 128. 
Momentum is consistently set to 0.9, and \(\gamma\) is chosen from 
\(\{0.1, 0.4, 0.95, 0.97\}\). 
The final hyperparameter values used in our main experiment results are summarized in Table~\ref{tab:finetune_hyper_params}.

\begin{table}[!ht]
\centering
\caption{
\textbf{hyperparameters.}
The below table shows the details we used to train our models. We reported the hyperparameters which gives the best predictive efficiency. We employed SGD optimizer for all training unless specified.}
\label{tab:finetune_hyper_params}
\resizebox{\textwidth}{!}{
\begin{tabular}{|c|c|c|c|c|c|c|c|c|c|}
\hline
Data                           & Architecture & Batch size & Epochs & $\eta$    & lr schedule & Momentum & weight decay & $\gamma$ & Temperature \\ \hline
\multirow{2}{*}{CIFAR-100}      & ResNet & 128        & 60   & 0.04  & 25, 40   & 0.9  & 0.0008  & 0.1 & 1.5       \\ \cline{2-10} 
& DenseNet & 64         & 60  & 0.025  & 25, 40    & 0.9 & 0.0004  & 0.1  & 1.8    \\ 
\hline
\multirow{2}{*}{Caltech-101} & ResNet  & 128       & 60  & 0.018  & 20, 35   & 0.9   & 0.0006   & 0.4   & 1.0       \\  \cline{2-10} 
 & DenseNet & 128        & 60  & 0.019 & 18, 30  & 0.9  & 0.0005   & 0.4   & 1.0       \\ 
 \hline 
\multirow{2}{*}{iNaturalist} & ResNet   & 128     & 60  & 0.0005 & 5    & 0.9   & 0.0008 & 0.95  & 1.5      \\ \cline{2-10} 
  & DenseNet  & 128     & 60  & 0.0002 & 5       & 0.9  & 0.0006  & 0.97  & 2.0      \\
 \hline
\end{tabular}
}
\end{table}

\vspace{1ex}
\noindent
\textbf{Evaluation results.} In our study, after the training phase for each method, we apply HPS (Equation~\ref{eq:HPS}), APS (Equation~\ref{eq:APS}), and RAPS (Equation~\ref{eq:RAPS}) during the calibration and testing stages.
We compute the marginal coverage and the APSS, reporting the mean and standard deviation across runs.
The results are summarized in Tables~\ref{tab:cvg_set_hps_train_hps_test}, \ref{tab:cvg_set_hps_train_aps_test}, and \ref{tab:cvg_set_hps_train_raps_test}.


\section{Additional Experiments}
\label{appendix:subsec:additional_exps}

\subsection{Additional Experiments for Marginal Coverage }
\label{appendix:subsec:additional_exps_marginal}

\begin{table*}[!ht]
\centering
\caption{
\textbf{Overall comparison on three datasets and calibrated by HPS on DenseNet and ResNet} with $\alpha = 0.1$.
All methods are evaluated under HPS calibration strategies to assess robustness across scoring functions.
We report the mean and standard deviation of the reported APSS and marginal coverage over 10 independent runs.
We benchmark four methods: standard CE, CUT, ConfTr, and  RWCE.
Arrows $\downarrow$ and $\uparrow$ indicate improvements or degradations in predictive efficiency relative to the best baseline.
Overall, RWCE consistently produces the smallest prediction sets across all datasets and evaluation metrics, demonstrating a relative improvement of $14.87\%$ in APSS—highlighting its superior calibration quality and generalization capability.
}
\label{tab:cvg_set_hps_train_hps_test}
\resizebox{\textwidth}{!}{
\begin{NiceTabular}{@{}c|cccc|cccc@{}}
\toprule
\multirow{2}{*}{Model} & \multicolumn{4}{c|}{Marginal Coverage} & \multicolumn{4}{c }{Prediction Set Size} \\ 
\cmidrule(lr){2-5} \cmidrule(lr){6-9}
& CE & CUT & ConfTr & RWCE & CE & CUT & ConfTr & RWCE  \\ 
\midrule
\Block{1-*}{Caltech-101}
\\
\midrule
ResNet   
& 0.90 $\pm$ 0.008 & 0.90 $\pm$ 0.005 & 0.90 $\pm$ 0.004 & 0.90 $\pm$ 0.008
& 1.52 $\pm$ 0.045 & 1.55 $\pm$ 0.039 & 1.26 $\pm$ 0.03 & \textbf{0.96 $\pm$ 0.008} ($\downarrow$ 23.81\%) 
\\ 
DenseNet 
& 0.90 $\pm$ 0.006 & 0.90 $\pm$ 0.006 & 0.90 $\pm$ 0.006 & 0.90 $\pm$ 0.004
& 3.51 $\pm$ 0.10  & 1.66 $\pm$ 0.038 & 3.13 $\pm$ 0.07  & \textbf{0.94 $\pm$ 0.005} ($\downarrow$ 43.37\%)  
\\ 
\midrule
\Block{1-*}{CIFAR-100}
\\
\midrule
ResNet   
& 0.90 $\pm$ 0.005 & 0.90 $\pm$ 0.006 & 0.90 $\pm$ 0.004 & 0.90 $\pm$ 0.006
& 3.39 $\pm$ 0.10  & 2.91 $\pm$ 0.08  & 3.98 $\pm$ 0.077  & \textbf{2.68 $\pm$ 0.083}  ($\downarrow$ 7.90\%)
\\ 
DenseNet 
& 0.90 $\pm$ 0.005 & 0.90 $\pm$ 0.006 & 0.90 $\pm$ 0.005 & 0.90 $\pm$ 0.006
& 2.59 $\pm$ 0.053 & 2.07 $\pm$ 0.06  & 2.19 $\pm$ 0.034  & \textbf{2.04 $\pm$ 0.051 }  ($\downarrow$ 1.45\%) 
\\ 
\midrule
\Block{1-*}{iNaturalist}
\\
\midrule
ResNet   
& 0.91 $\pm$ 0.017 & 0.90 $\pm$ 0.009 & 0.90 $\pm$ 0.012 & 0.91 $\pm$ 0.007
& 98.69 $\pm$ 8.86 & 73.19 $\pm$ 3.13 & 76.30 $\pm$ 3.12 & \textbf{69.05 $\pm$ 2.56} ($\downarrow$ 5.66\%) 
\\ 
DenseNet 
& 0.90 $\pm$ 0.008 & 0.90 $\pm$ 0.012 & 0.90 $\pm$ 0.010 & 0.90 $\pm$ 0.011
& 93.18 $\pm$ 2.92 & 74.11 $\pm$ 3.35 & 72.25 $\pm$ 2.07 & \textbf{67.17 $\pm$ 2.07} ($\downarrow$ 7.03\%) 
\\ 
\bottomrule
\end{NiceTabular}
}
\end{table*}

\begin{table*}[!t]
\centering
\caption{
\textbf{Overall comparison on three datasets and calibrated by APS on DenseNet and ResNet} with $\alpha = 0.1$.
All methods are evaluated under APS calibration strategies to assess robustness across scoring functions.
We report the mean and standard deviation of the reported APSS and marginal coverage over 10 independent runs.
We benchmark four methods: standard CE, CUT, ConfTr, and  RWCE.
Arrows $\downarrow$ and $\uparrow$ indicate improvements or degradations in predictive efficiency relative to the best baseline.
Overall, RWCE consistently produces the smallest prediction sets across all datasets and evaluation metrics, demonstrating a relative improvement of $27.89\%$ in APSS—highlighting its superior calibration quality and generalization capability.
}
\label{tab:cvg_set_hps_train_aps_test}
\resizebox{\textwidth}{!}{
\begin{NiceTabular}{@{}c|cccc|cccc@{}}
\toprule
\multirow{2}{*}{Model} & \multicolumn{4}{c|}{Marginal Coverage} & \multicolumn{4}{c }{Prediction Set Size} \\ 
\cmidrule(lr){2-5} \cmidrule(lr){6-9}
& CE & CUT & ConfTr & RWCE & CE & CUT & ConfTr & RWCE  \\ 
\midrule
\Block{1-*}{Caltech-101}
\\
\midrule
ResNet   
& 0.90 $\pm$ 0.003 & 0.90 $\pm$ 0.005 & 0.90 $\pm$ 0.004 & 0.90 $\pm$ 0.006
& 4.96 $\pm$ 0.094 & 4.89 $\pm$ 0.095 & 4.25 $\pm$ 0.081 & \textbf{1.33 $\pm$ 0.017} ($\downarrow$ 68.71\%) 
\\ 
DenseNet 
& 0.90 $\pm$ 0.006 & 0.90 $\pm$ 0.006 & 0.90 $\pm$ 0.007 & 0.90 $\pm$ 0.004
& 8.9  $\pm$ 0.18 & 4.60 $\pm$ 0.078  & 9.69 $\pm$ 0.22  & \textbf{1.27 $\pm$ 0.008} ($\downarrow$ 72.39\%) 
\\ 
\midrule
\Block{1-*}{CIFAR-100}
\\
\midrule
ResNet   
& 0.90 $\pm$ 0.006 & 0.90 $\pm$ 0.05   & 0.90 $\pm$ 0.005 & 0.90 $\pm$ 0.006 
& 3.98 $\pm$0.13 & 3.49 $\pm$ 0.104 & 5.13 $\pm$ 0.117  & \textbf{3.08 $\pm$ 0.078} ($\downarrow$ 11.75\%)
\\ 
DenseNet 
& 0.90 $\pm$ 0.006 & 0.90 $\pm$ 0.006 & 0.90 $\pm$ 0.007 & 0.90 $\pm$ 0.007
& 3.38 $\pm$ 0.12 & \textbf{2.19 $\pm$ 0.060} & 3.04 $\pm$ 0.069   & 2.335 $\pm$ 0.076 ($\uparrow$ 6.62\%)
\\ 
\midrule
\Block{1-*}{iNaturalist}
\\
\midrule
ResNet   
& 0.90 $\pm$ 0.011 & 0.90 $\pm$ 0.009 & 0.90 $\pm$ 0.012 & 0.91 $\pm$ 0.009
& 95.18 $\pm$ 3.50 & 79.58 $\pm$ 2.87 & 87.80 $\pm$ 1.97   & \textbf{73.39 $\pm$ 2.32} ($\downarrow$ 7.78\%) 
\\ 
DenseNet 
& 0.90 $\pm$ 0.008 & 0.90 $\pm$ 0.014 & 0.90 $\pm$ 0.012 & 0.90 $\pm$ 0.011
& 101.55 $\pm$ 3.16 & 87.27 $\pm$ 2.30 & 92.88 $\pm$ 2.89 & \textbf{75.65 $\pm$ 2.44} ($\downarrow$ 13.31\%)
\\ 
\bottomrule
\end{NiceTabular}
}
\end{table*}

\begin{table*}[!t]
\centering
\caption{
\textbf{Overall comparison on three datasets and calibrated by RAPS on DenseNet and ResNet} with $\alpha = 0.1$.
All methods are evaluated under RAPS calibration strategies to assess robustness across scoring functions.
We report the mean and standard deviation of the reported APSS and marginal coverage over 10 independent runs.
We benchmark four methods: standard CE, CUT, ConfTr, and  RWCE.
Arrows $\downarrow$ and $\uparrow$ indicate improvements or degradations in predictive efficiency relative to the best baseline.
Overall, RWCE consistently produces the smallest prediction sets across all datasets and evaluation metrics, demonstrating a relative improvement of $23.79\%$ in APSS—highlighting its superior calibration quality and generalization capability.
}
\label{tab:cvg_set_hps_train_raps_test}
\resizebox{\textwidth}{!}{
\begin{NiceTabular}{@{}c|cccc|cccc@{}}
\toprule
\multirow{2}{*}{Model} & \multicolumn{4}{c|}{Marginal Coverage} & \multicolumn{4}{c }{Prediction Set Size} \\ 
\cmidrule(lr){2-5} \cmidrule(lr){6-9}
& CE & CUT & ConfTr & RWCE & CE & CUT & ConfTr & RWCE  \\ 
\midrule
\Block{1-*}{Caltech-101}
\\
\midrule
ResNet   
& 0.90 $\pm$ 0.003 & 0.90 $\pm$ 0.004 & 0.90 $\pm$ 0.005 & 0.90 $\pm$ 0.006
& 3.83 $\pm$ 0.061 & 4.13 $\pm$ 0.068 & 3.58 $\pm$ 0.066 & \textbf{1.32 $\pm$ 0.016} ($\downarrow$ 63.13\%) 
\\ 
DenseNet 
& 0.90 $\pm$ 0.007 & 0.90 $\pm$ 0.007 & 0.90 $\pm$ 0.008 & 0.90 $\pm$ 0.004
& 6.57  $\pm$ 0.12 & 3.65 $\pm$ 0.09  & 6.77 $\pm$ 0.14  & \textbf{1.26 $\pm$ 0.008} ($\downarrow$ 65.48\%) 
\\ 
\midrule
\Block{1-*}{CIFAR-100}
\\
\midrule
ResNet   
& 0.90 $\pm$ 0.007 & 0.90 $\pm$ 0.07 & 0.90 $\pm$ 0.005 & 0.90 $\pm$ 0.007 
& 3.25 $\pm$ 0.13 & 2.92 $\pm$ 0.07 & 4.08 $\pm$ 0.07  & \textbf{2.77 $\pm$ 0.065} ($\downarrow$ 5.14\%) 
\\ 
DenseNet 
& 0.90 $\pm$ 0.005 & 0.90 $\pm$ 0.005 & 0.90 $\pm$ 0.007 & 0.90 $\pm$ 0.007
& 2.73 $\pm$ 0.043 & \textbf{2.01 $\pm$ 0.039} & 2.74 $\pm$ 0.045   & 2.13 $\pm$ 0.043 ($\uparrow$ 5.97\%)
\\ 
\midrule
\Block{1-*}{iNaturalist}
\\
\midrule
ResNet   
& 0.90 $\pm$ 0.008 & 0.90 $\pm$ 0.009 & 0.90 $\pm$ 0.009 & 0.90 $\pm$ 0.007
& 98.72 $\pm$ 2.63 & 81.47 $\pm$ 3.22 & 82.39 $\pm$ 3.50   & \textbf{79.52 $\pm$ 1.90} ($\downarrow$ 2.39\%) 
\\ 
DenseNet 
& 0.89 $\pm$ 0.007 & 0.90 $\pm$ 0.011 & 0.90 $\pm$ 0.013 & 0.90 $\pm$ 0.007
& 95.48 $\pm$ 2.30 & 81.67 $\pm$ 3.89 & 84.51 $\pm$ 4.33 & \textbf{71.93 $\pm$ 1.95} ($\downarrow$ 12.59\%)
\\ 
\bottomrule
\end{NiceTabular}
}
\end{table*}

\begin{table*}[!ht]
\centering
\caption{
\textbf{Overall comparison on three datasets and calibrated by SAPS on DenseNet and ResNet} with $\alpha = 0.1$.
All methods are evaluated under SAPS calibration strategies to assess robustness across scoring functions.
We report the mean and standard deviation of the reported APSS and marginal coverage over 10 independent runs.
We benchmark four methods: standard CE, CUT, ConfTr, and  RWCE.
Arrows $\downarrow$ and $\uparrow$ indicate improvements or degradations in predictive efficiency relative to the best baseline.
Overall, RWCE consistently produces the smallest prediction sets across all datasets and evaluation metrics, demonstrating a relative improvement of $20.61\%$ in APSS—highlighting its superior calibration quality and generalization capability.
}
\label{tab:cvg_set_hps_train_saps_test}
\resizebox{\textwidth}{!}{
\begin{NiceTabular}{@{}c|cccc|cccc@{}}
\toprule
\multirow{2}{*}{Model} & \multicolumn{4}{c|}{Marginal Coverage} & \multicolumn{4}{c }{Prediction Set Size} \\ 
\cmidrule(lr){2-5} \cmidrule(lr){6-9}
& CE & CUT & ConfTr & RWCE & CE & CUT & ConfTr & RWCE  \\ 
\midrule
\Block{1-*}{Caltech-101}
\\
\midrule
ResNet   
& 0.90 $\pm$ 0.007 & 0.90 $\pm$ 0.006 & 0.90 $\pm$ 0.005 & 0.90 $\pm$ 0.006
& 2.50 $\pm$ 0.033 & 2.51 $\pm$ 0.032 & 2.45 $\pm$ 0.02 & \textbf{1.34 $\pm$ 0.019} ($\downarrow$ 45.36\%) 
\\ 
DenseNet 
& 0.90 $\pm$ 0.006 & 0.90 $\pm$ 0.007 & 0.90 $\pm$ 0.007 & 0.90 $\pm$ 0.006
& 3.84 $\pm$ 0.18  & 2.43 $\pm$ 0.041 & 4.75 $\pm$ 0.26  & \textbf{1.30 $\pm$ 0.014} ($\downarrow$ 46.70\%)  
\\ 
\midrule
\Block{1-*}{CIFAR-100}
\\
\midrule
ResNet   
& 0.90 $\pm$ 0.006 & 0.90 $\pm$ 0.006 & 0.90 $\pm$ 0.005 & 0.90 $\pm$ 0.007
& 3.98 $\pm$ 0.13  & 3.44 $\pm$ 0.11  & 4.59 $\pm$ 0.13  & \textbf{3.17 $\pm$ 0.12}  ($\downarrow$ 7.73\%)
\\ 
DenseNet 
& 0.90 $\pm$ 0.006 & 0.90 $\pm$ 0.006 & 0.90 $\pm$ 0.005 & 0.90 $\pm$ 0.005
& 2.97 $\pm$ 0.12 & 2.63 $\pm$ 0.078  & 2.41 $\pm$ 0.046  & \textbf{2.39 $\pm$ 0.066 }  ($\downarrow$ 9.16\%) 
\\ 
\midrule
\Block{1-*}{iNaturalist}
\\
\midrule
ResNet   
& 0.90 $\pm$ 0.008 & 0.90 $\pm$ 0.009 & 0.90 $\pm$ 0.008 & 0.90 $\pm$ 0.010
& 98.69 $\pm$ 2.84 & 81.62 $\pm$ 3.19 & 83.73 $\pm$ 2.66 & \textbf{79.16 $\pm$ 2.89} ($\downarrow$ 3.00\%) 
\\ 
DenseNet 
& 0.90 $\pm$ 0.006 & 0.90 $\pm$ 0.010 & 0.90 $\pm$ 0.013 & 0.90 $\pm$ 0.008
& 95.98 $\pm$ 2.00 & 82.40 $\pm$ 3.48 & 85.67 $\pm$ 3.42 & \textbf{72.74 $\pm$ 2.41} ($\downarrow$ 11.72\%) 
\\ 
\bottomrule
\end{NiceTabular}
}
\end{table*}

\noindent 
\textbf{RWCE generates smaller prediction sets.}
Table~\ref{tab:cvg_set_hps_train_hps_test}, Table~\ref{tab:cvg_set_hps_train_aps_test},  Table~\ref{tab:cvg_set_hps_train_raps_test}, and Table~\ref{tab:cvg_set_hps_train_saps_test} report the APSS and marginal coverage rates of all methods across three datasets under a fixed coverage level of $\alpha = 0.1$. 
All models are calibrated using the HPS score during the training phase and are evaluated with HPS, APS, RAPS and SAPS scoring functions during the evaluation phase, with further details provided in Appendix~\ref{subsec:train_test_strategies}.
Our method, RWCE, demonstrates consistently strong performance across all datasets and evaluation settings by producing smaller and more efficient prediction sets than existing baselines.
Averaging the relative improvements across all configurations, RWCE achieves a 21.79\% reduction in prediction set size compared to the strongest baseline in each case.
Table~\ref{tab:cvg_set_hps_train_hps_test} presents the APSS and marginal coverage rates for different methods when using the HPS score across training, calibration, and testing phases.  
RWCE outperforms all existing baselines, achieving an average of 14.87\% reduction in prediction set size across all datasets.
Table~\ref{tab:cvg_set_hps_train_aps_test} reports results when HPS is used during training and APS is used for calibration and evaluation.  
RWCE continues to outperform nearly all baselines, achieving a 27.89\% average reduction in prediction set size, except for an increase of 6.62\% on CIFAR-100 with DenseNet.
Table~\ref{tab:cvg_set_hps_train_raps_test} summarizes the performance under the RAPS scoring function, with HPS used during training.  
RWCE achieves a 23.79\% average reduction in prediction set size, again demonstrating competitive or superior performance in most settings.  
The only notable exception is CIFAR-100 with DenseNet, where RWCE yields an increase of 5.97\% compared to the strongest baseline.
Table~\ref{tab:cvg_set_hps_train_saps_test} presents results under the SAPS scoring function with HPS training.
RWCE attains an average 20.61 \% reduction in prediction set size, demonstrating robust and generally superior performance across diverse scenarios.

\begin{figure*}[!t]
    \centering
    \begin{minipage}[t]{0.30\linewidth}
    \centering
    \textbf{(a)}  Loss Convergence
    \includegraphics[width = \linewidth]{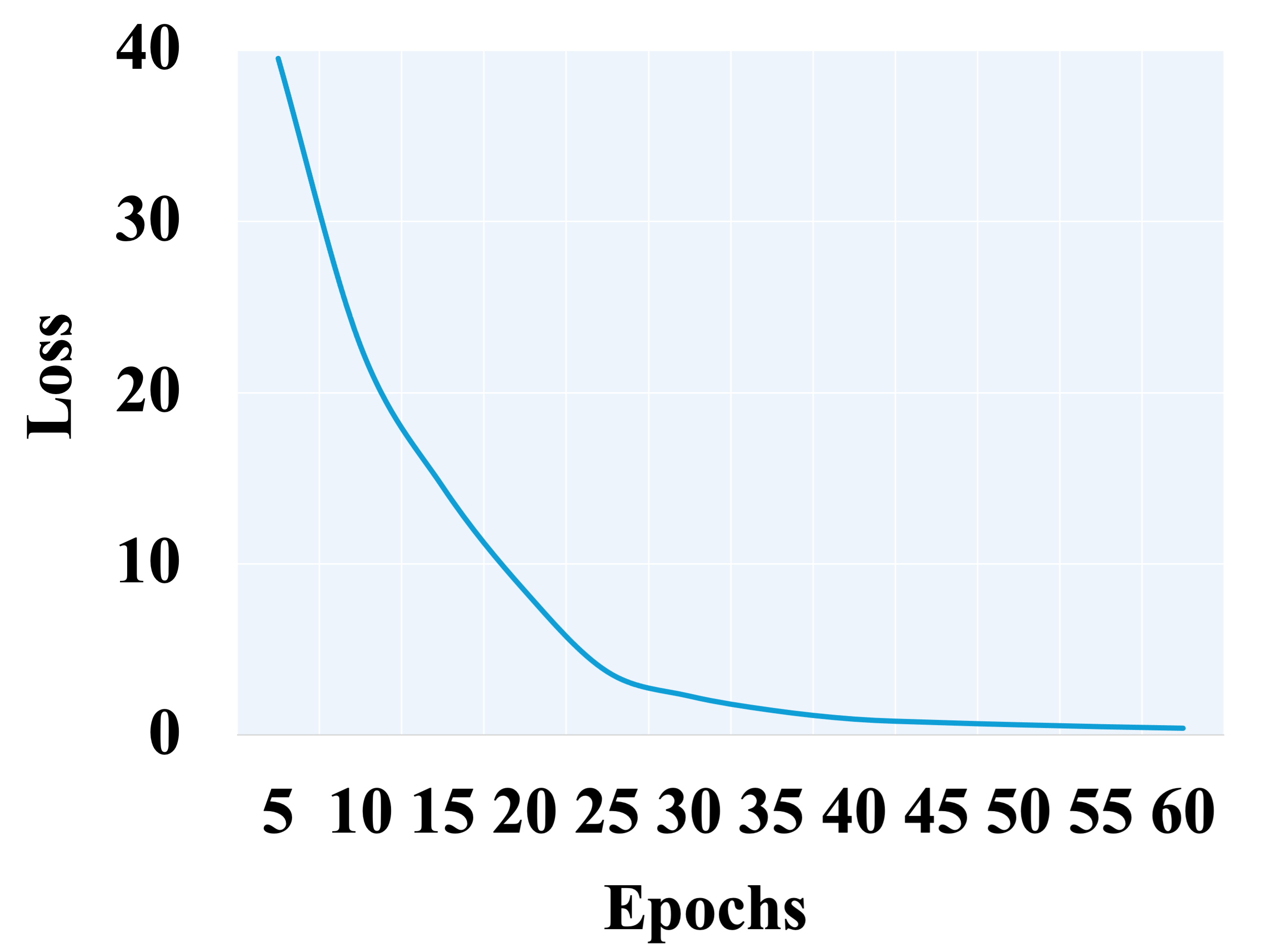}
    \\
    \includegraphics[width = \linewidth]{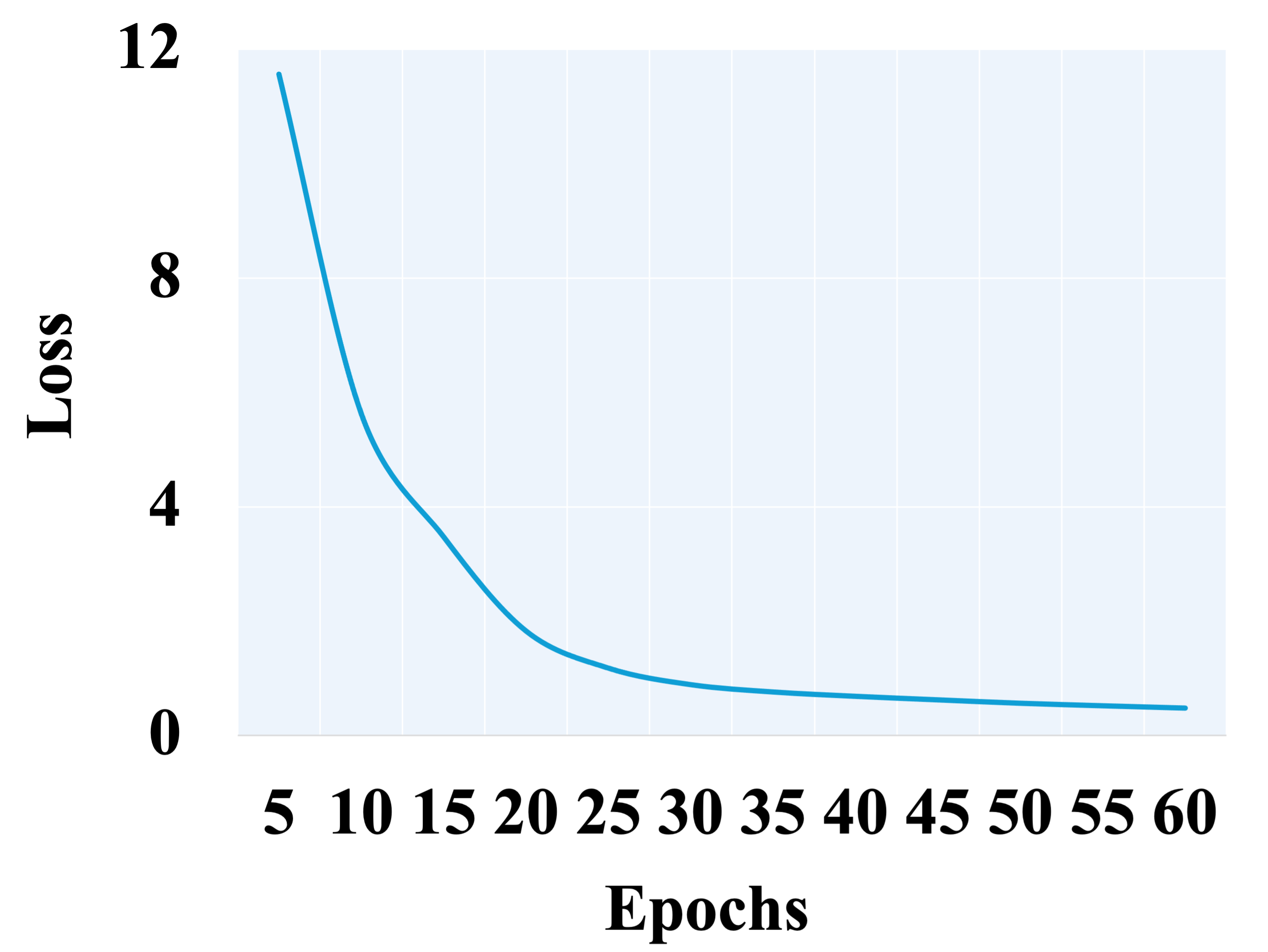}
    \end{minipage} 
    \begin{minipage}[t]{0.31\linewidth}
    
    \centering    
    \textbf{(b)} Loss vs APSS
    \includegraphics[width=\linewidth]{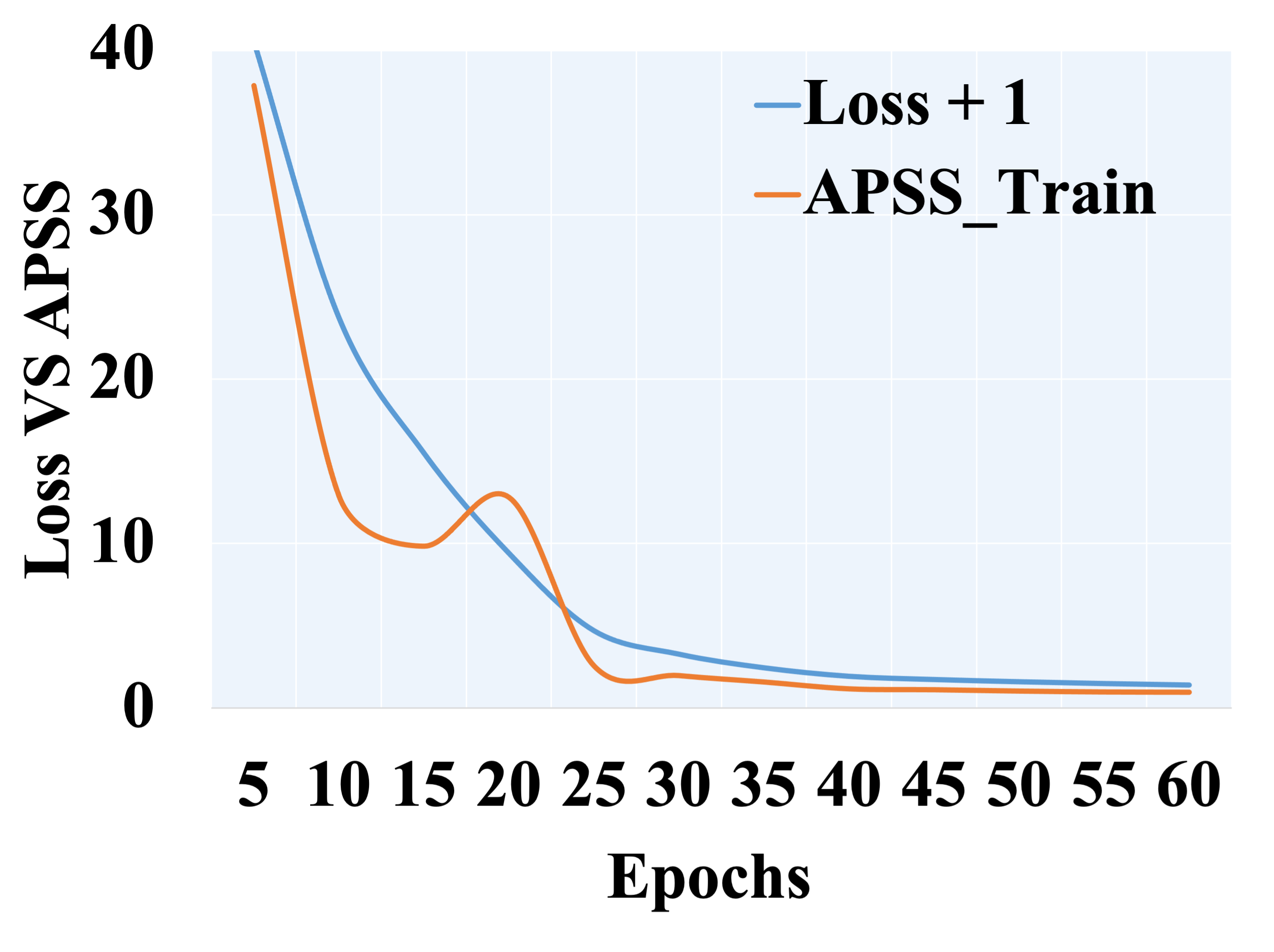}
    \\
    \includegraphics[width=\linewidth]{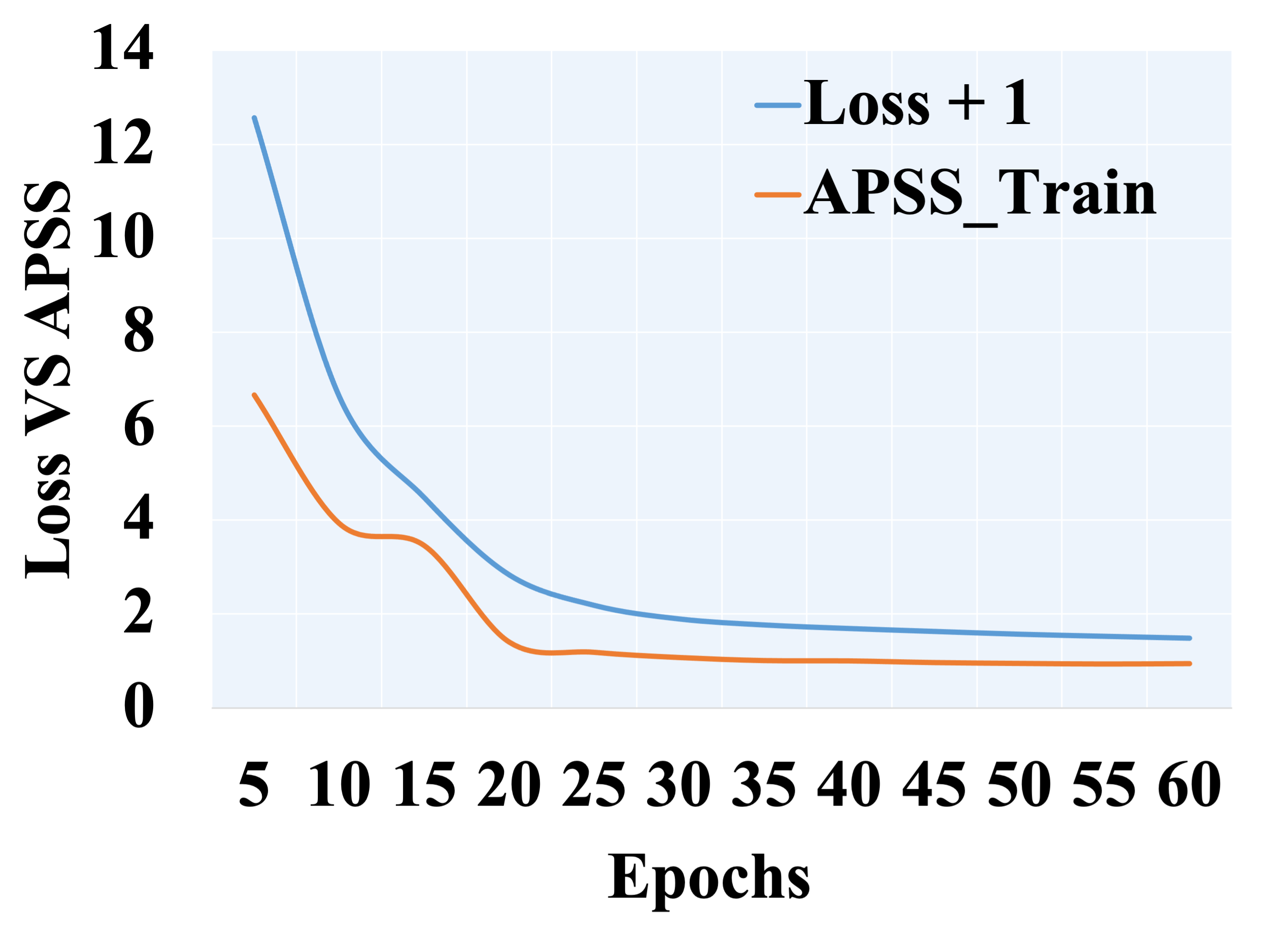}
    \end{minipage} 
    \begin{minipage}[t]{0.32\linewidth}
    \centering
    \textbf{(c)}    APSS Comparison 
     \includegraphics[width=\linewidth]{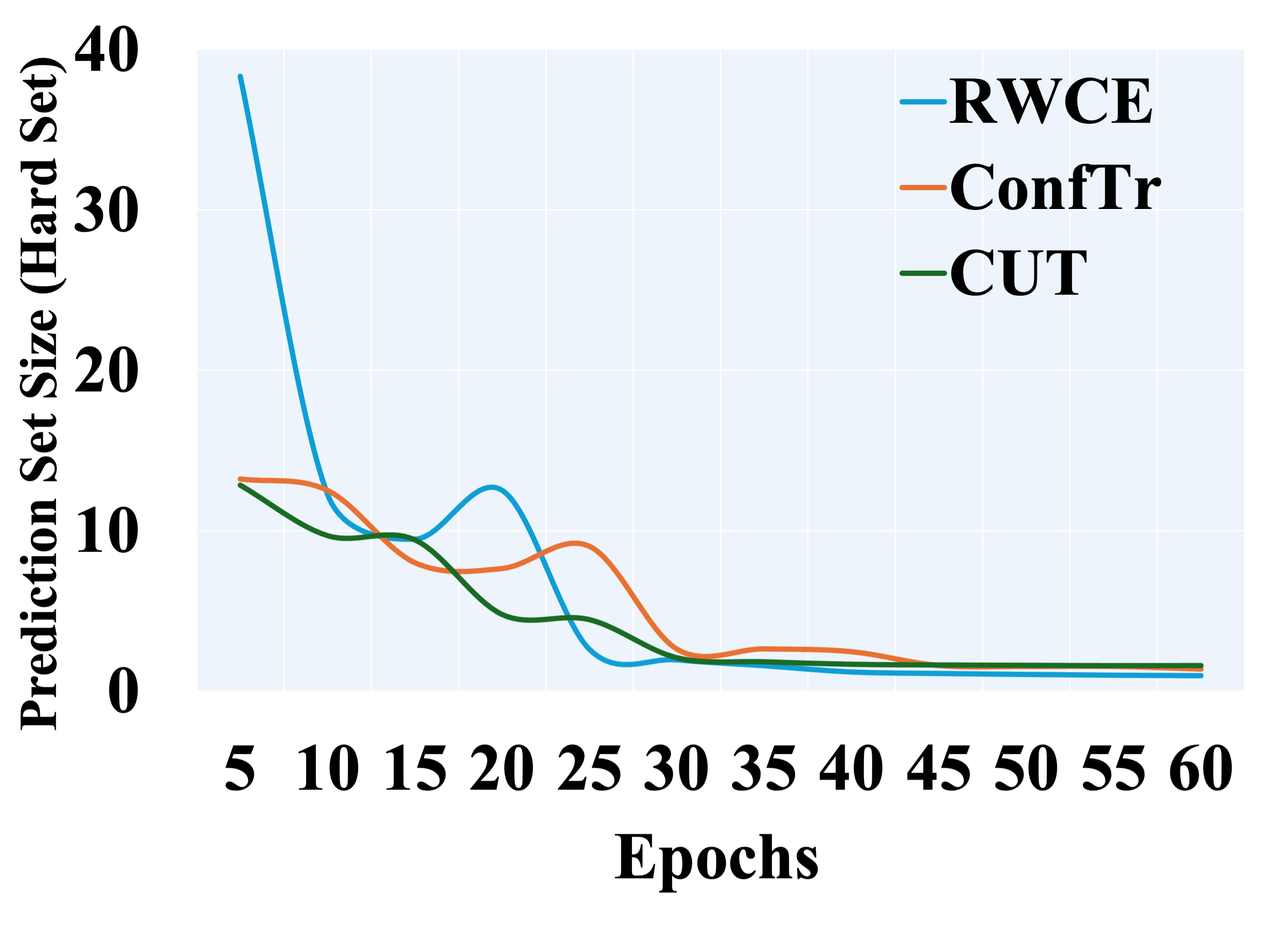}
    \\   
    \includegraphics[width=\linewidth]{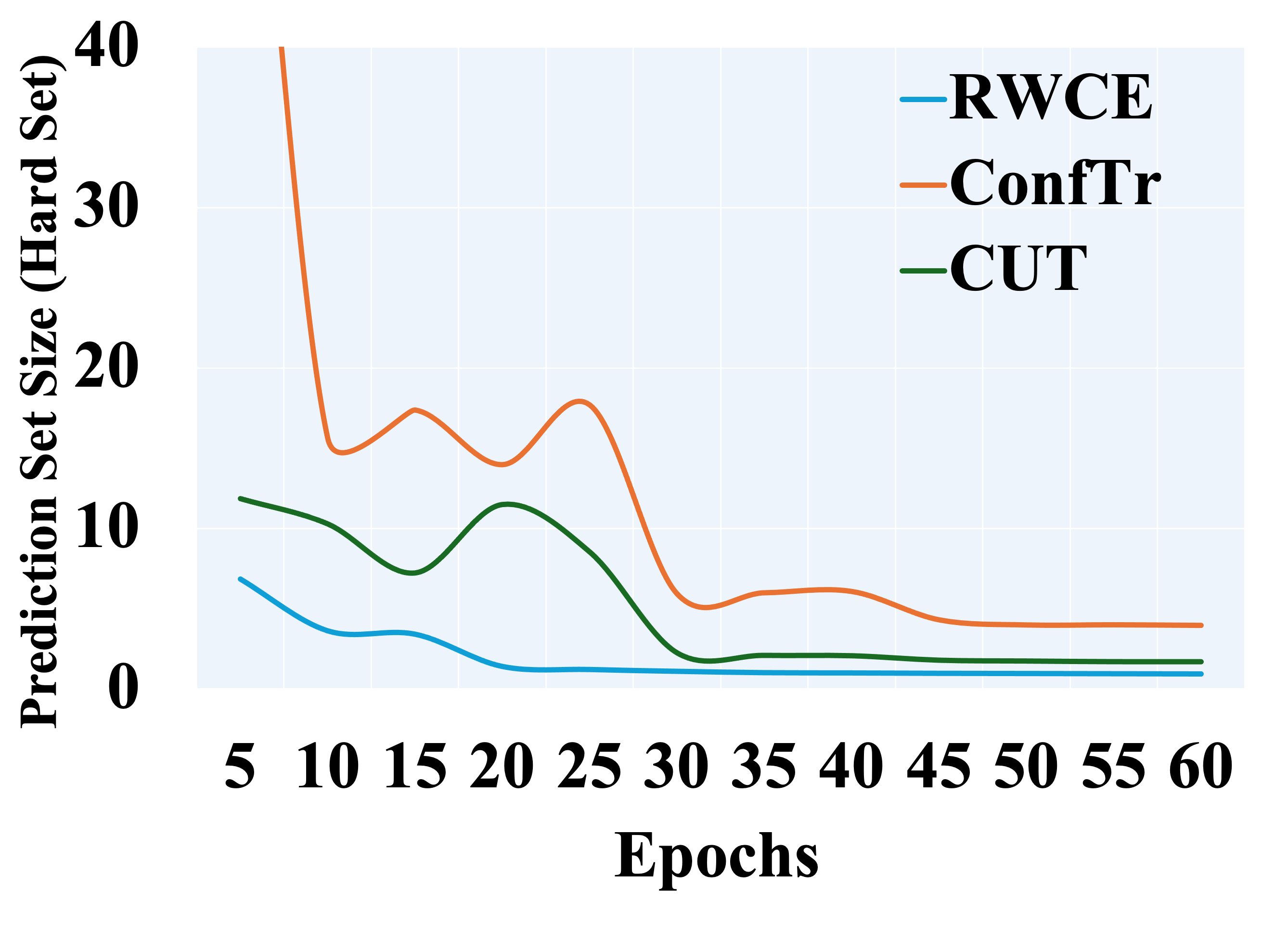}
    \end{minipage} 
    \caption{
    \textbf{Justification experiments using ResNet (Top row) and DenseNet (Bottom row)} on Caltech-101.
    \textbf{(a)} the training loss convergence of RWCE . The results demonstrate that RWCE converges smoothly and stably on both architectures.  
    \textbf{(b)} the RWCE training loss inflated with $1$ (according to remark~\ref{remark:new_objective_bounds_rank}) with the actual APSS on the validation set. 
    The loss closely upper bounds APSS with a small and stable gap, indicating that RWCE effectively approximates and directly minimizes the true set size objective.  
    \textbf{(c)} the APSS of RWCE, ConfTr, and CUT calibrated by HPS score. RWCE consistently achieves smaller prediction sets on both architectures, confirming its superior efficiency in directly minimizing set size.
    }
    \label{fig:results_overall_Caltech}
\end{figure*}

\begin{figure}[!t]
    \centering
    \begin{minipage}[t]{0.49\linewidth}
    \centering
    \textbf{(a)} Alignment Inequality on ResNet
    \end{minipage} 
    \begin{minipage}[t]{0.49\linewidth}
    \centering
    \textbf{(b)} Alignment Inequality on DenseNet
    \end{minipage} 
    \hfill
    \begin{minipage}[t]{0.49\linewidth}
     \centering   
     \includegraphics[width = \linewidth]{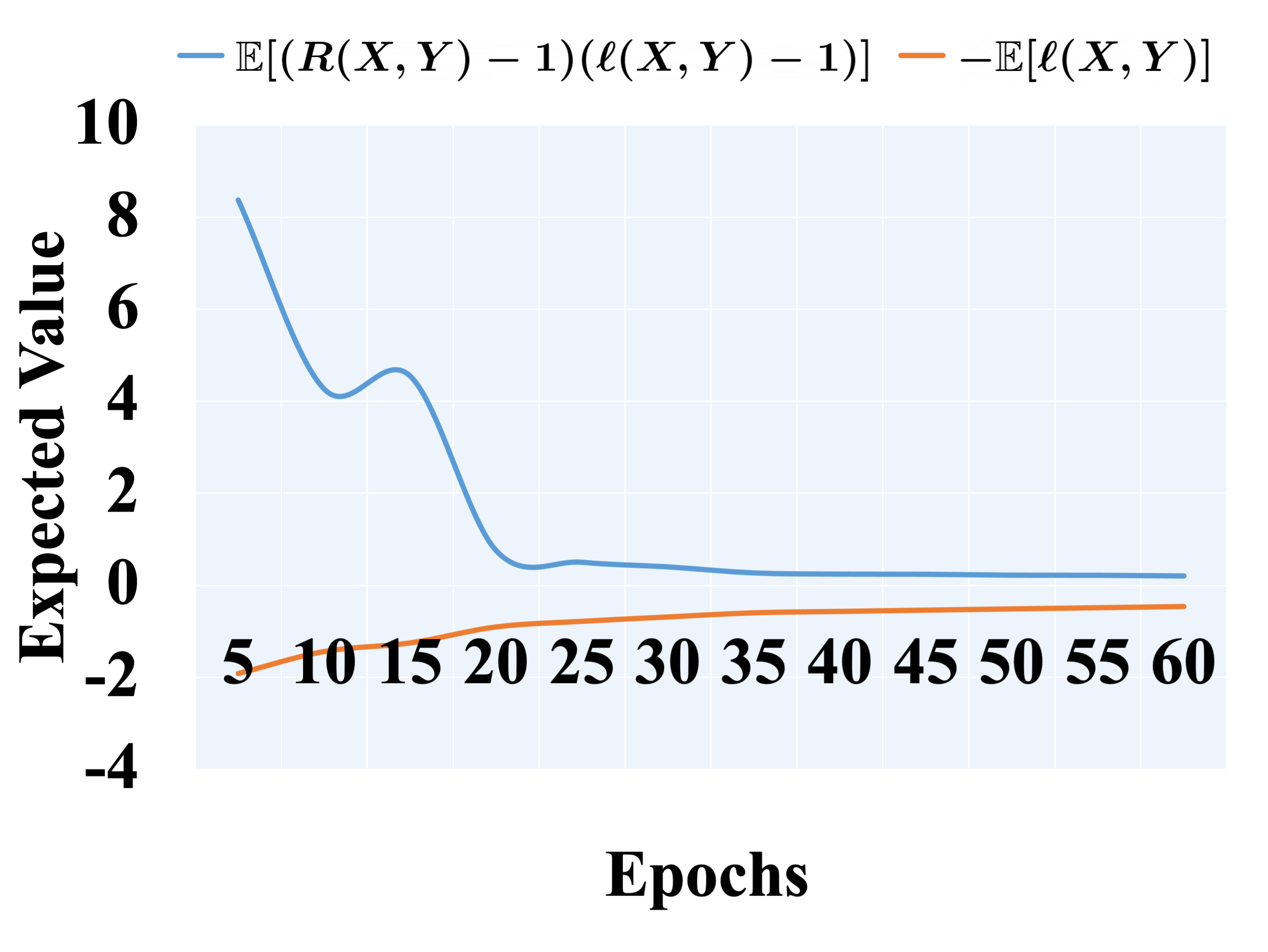}
     \end{minipage}
    \begin{minipage}[t]{0.49\linewidth}
    \centering
    \includegraphics[width=\linewidth]{RWCE_Rank_Loss_caltech_dense.png}
    \end{minipage}
    \caption{
    \textbf{
    Empirical validation of the alignment assumption in Theorem~\ref{theorem:new_objective_upper_bounds_expected_rank}, which states that $\mathbb{E}[(R(X,Y) - 1)(\ell(X,Y) - 1)] \geq -\mathbb{E}[\ell(X,Y)]$} during training on Caltech-101 using ResNet (a) and DenseNet (b).  
    We plot the left-hand side (expected rank–loss interaction term, blue) and right-hand side (negative expected cross-entropy loss, orange). 
    In both cases, the blue curve remains significantly above the orange curve throughout training, confirming that the inequality holds in practice.
    }
    \label{fig:rank_ce_alignment_caltech}
    \vspace{-2.0ex}
\end{figure}

\vspace{1ex}
\noindent
\textbf{RWCE Converges Stably During Training.}  
To assess the optimization dynamics of RWCE, we visualize its training loss (according to remark~\ref{remark:new_objective_bounds_rank}) over 60 fine-tuning epochs on Caltech-101 using both ResNet and DenseNet architectures, as shown in Figure~\ref{fig:results_overall_Caltech}(a). 
The training curves exhibit a consistent convergence pattern across both models. The loss decreases rapidly during the first 20 epochs, followed by a more gradual decline. In both models, convergence begins to stabilize around epoch 30, and the loss becomes stable after epoch 40.
This loss corresponds to an rank-weighted objective based on the rank of the true label, which we theoretically show to be a tight upper bound on the expected prediction set size. 
The observed convergence behavior thus not only indicates optimization stability, but also reflects RWCE's ability to directly approximate the set size minimization objective without relying on relaxed approximations of the indicator function.

\vspace{1ex}

\noindent
\textbf{RWCE Directly Minimizes Prediction Set Size.}  
To examine whether RWCE effectively minimizes the true prediction set size on Caltech-101, we compare its training loss with the actual APSS on the training set, as shown in Figure~\ref{fig:results_overall_Caltech}(b). For both ResNet and DenseNet, the training loss inflated with $1$ (according to remark~\ref{remark:new_objective_bounds_rank}) closely upper bounds the APSS throughout training, with the two curves converging to similar values after epoch 40. This tight alignment empirically validates our theoretical claim that the RWCE objective serves as a provable upper bound on the expected set size.
Furthermore, we track the test-time APSS of RWCE, ConfTr, and CUT under HPS calibration during training in Figure~\ref{fig:results_overall_Caltech}(c). 
RWCE effectively reduces prediction set size throughout training. On the ResNet backbone (top), it starts with relatively large sets but overtakes both ConfTr and CUT by around epoch 30, maintaining the smallest APSS thereafter. On the DenseNet backbone (bottom), RWCE consistently achieves the lowest prediction set size from the very beginning. ConfTr exhibits relatively smooth convergence on ResNet, but its performance on DenseNet is unstable in the early stages and remains inferior to RWCE throughout. CUT performs competitively in the early epochs on both architectures but is eventually outperformed by RWCE. These results highlight RWCE’s robustness and efficiency in minimizing prediction set size across different model architectures on Caltech-101.

\newpage
\subsection{Additional Experiments on NLP Dataset with Transformer Architecture}
\label{appendix:subsec:additional_exps_nlp}

\noindent

\noindent
\textbf{RWCE Generates Smaller Prediction Sets on NLP Tasks.}  
To further assess the generality of RWCE beyond vision task, we evaluate it on the fine-grained sentiment classification dataset SST-5 ~\citep{socher2013recursive} using the BERT transformer \citep{devlin2019bert}.
As summarized in Table~\ref{tab:nlp_set_hps_train_test}, RWCE achieves an overall 3.90\% reduction in the average prediction set size compared to baseline, ConfTr and CUT across all calibration scores (HPS, APS, RAPS, SAPS), while maintaining comparable marginal coverage around $0.90$. 
Overall, RWCE generalizes robustly across domains—retaining its advantage in minimizing prediction set size and improving predictive efficiency not only on image classification tasks but also in NLP applications built on large transformer models.

\begin{table*}[!ht]
\centering
\caption{
\textbf{Overall comparison on SST-5 and calibrated by all scores on BERT} with $\alpha = 0.1$.
All methods are evaluated under calibration strategies to assess robustness across scoring functions.
We report the mean and standard deviation of the reported APSS and marginal coverage over 10 independent runs.
We benchmark four methods: standard CE, CUT, ConfTr, and  RWCE.
Arrows $\downarrow$ and $\uparrow$ indicate improvements or degradations in predictive efficiency relative to the best baseline.
Overall, RWCE consistently produces the smallest prediction sets across all evaluation metrics, demonstrating a relative improvement of $3.90\%$ in APSS—highlighting its superior calibration quality and generalization capability.
}
\label{tab:nlp_set_hps_train_test}
\resizebox{\textwidth}{!}{
\begin{NiceTabular}{@{}c|cccc|cccc@{}}
\toprule
\multirow{2}{*}{Model} & \multicolumn{4}{c|}{Prediction Set Size} & \multicolumn{4}{c }{Marginal Coverage} \\ 
\cmidrule(lr){2-5} \cmidrule(lr){6-9}
& CE & CUT & ConfTr & RWCE & CE & CUT & ConfTr & RWCE  \\ 
\midrule
\Block{1-*}{SST-5}
\\
\midrule
HPS   
& 3.88 $\pm$ 0.083 & 2.52 $\pm$ 0.073 & 2.62 $\pm$ 0.047 & \textbf{2.39 $\pm$ 0.078} ($\downarrow$ 5.16\%)
& 0.90 $\pm$ 0.017 & 0.90 $\pm$ 0.020 & 0.90 $\pm$ 0.014 & 0.89 $\pm$ 0.020  
\\ 
APS 
& 4.19 $\pm$ 0.067 & 2.59 $\pm$ 0.057 & 2.71 $\pm$ 0.050  & \textbf{2.52 $\pm$ 0.051} ($\downarrow$ 2.70\%)
& 0.90 $\pm$ 0.018 & 0.90 $\pm$ 0.015 & 0.90 $\pm$ 0.017 & 0.90 $\pm$ 0.021  
\\ 
RAPS 
& 4.56 $\pm$ 0.323  & 2.61 $\pm$ 0.050 & 2.66 $\pm$ 0.055  & \textbf{2.55 $\pm$ 0.045} ($\downarrow$ 2.30\%)
& 0.94 $\pm$ 0.045 & 0.90 $\pm$ 0.016 & 0.90 $\pm$ 0.018 & 0.90 $\pm$ 0.014  
\\ 
SAPS 
& 4.17 $\pm$ 0.092 & 2.88 $\pm$ 0.080 & 2.79 $\pm$ 0.053  & \textbf{2.64 $\pm$ 0.061} ($\downarrow$ 5.38\%)
& 0.90 $\pm$ 0.021 & 0.90 $\pm$ 0.019 & 0.90 $\pm$ 0.013 & 0.90 $\pm$ 0.021  
\\ 
\bottomrule
\end{NiceTabular}
}
\end{table*}

\vspace{1ex}

\end{document}